\documentclass{article}[12pt]

\usepackage[margin=1.1in]{geometry}

\usepackage[utf8]{inputenc} 
\usepackage[T1]{fontenc}    
\usepackage{hyperref}       
\usepackage{url}            
\usepackage{booktabs}       
\usepackage{amsfonts}       
\usepackage{nicefrac}       
\usepackage{microtype}      
\usepackage{xcolor}         
\usepackage{amsthm,amssymb,amsmath}
\usepackage{graphicx}
\usepackage{subcaption}
\usepackage{wrapfig}

\usepackage{natbib}

\usepackage{xcolor}

\usepackage{enumitem}

\theoremstyle{plain}
\newtheorem{theorem}{Theorem}[section]
\newtheorem{proposition}[theorem]{Proposition}
\newtheorem{lemma}[theorem]{Lemma}
\newtheorem{corollary}[theorem]{Corollary}
\theoremstyle{definition}
\newtheorem{definition}[theorem]{Definition}
\newtheorem{assumption}[theorem]{Assumption}
\newtheorem{condition}[theorem]{Condition}
\theoremstyle{remark}
\newtheorem{remark}[theorem]{Remark}

\usepackage{amsmath,amsfonts,bm}
\usepackage{mathtools}

\def\eqref#1{equation~\ref{#1}}

\def\1{\bm{1}}

\def\rva{{\mathbf{a}}}
\def\rvb{{\mathbf{b}}}

\def\rvf{{\mathbf{f}}}

\def\rvr{{\mathbf{r}}}

\def\rvv{{\mathbf{v}}}
\def\rvw{{\mathbf{w}}}
\def\rvx{{\mathbf{x}}}
\def\rvy{{\mathbf{y}}}

\DeclareMathAlphabet{\mathsfit}{\encodingdefault}{\sfdefault}{m}{sl}
\SetMathAlphabet{\mathsfit}{bold}{\encodingdefault}{\sfdefault}{bx}{n}

\title{Label-NTK Alignments and A Tighter Convergence Bound in the NTK Regime}

\author{%
  Ruchirinkil Marreddy ~~~~~~~~~~~~~~~ Chaoyue Liu 
    \\
    \\
  Elmore Family School of Electrical and Computer Engineering\\
  Purdue University\\
  \texttt{\{rmarred, cyliu\}@purdue.edu} \\ 
}

\begin{document}

\maketitle

\begin{abstract}
The Neural Tangent Kernel (NTK) framework explains optimization in over-parameterized neural networks via approximately linearized dynamics, yielding exponential convergence guarantees. However, existing results are often overly pessimistic and do not match the fast training in practice, as they depend on the smallest NTK eigenvalue, which is typically extremely small in practice.
In this work, we develop sharper convergence guarantees by characterizing the interaction between data labels and the NTK eigen-spectrum. We identify two key phenomena, Label-NTK alignment and Residual-NTK alignment, showing that projections of labels and residuals onto NTK eigenvectors scale with the corresponding eigenvalues. We provide empirical evidence and theoretical justification under mild data assumptions.
Exploiting these alignment properties, we derive a refined convergence bound that depends on the full spectrum and closely matches practical training dynamics, significantly improving over classical worst-case results. We further obtain improved generalization bounds. Experiments on MLPs and CNNs across multiple datasets validate our theory.
\end{abstract}

\section{Introduction}

Understanding why deep neural networks can be efficiently optimized by simple algorithms such as gradient descent, despite their highly non-convex objectives, remains a central question in modern machine learning. In recent years, a prominent line of work has made significant progress by studying over-parameterized neural networks through the lens of the Neural Tangent Kernel (NTK) \cite{jacot2018ntk}. In the regime of sufficient over-parameterization, neural networks exhibit approximately linear training dynamics \cite{lee2019wide, liu2020linearity}, which has enabled a rich body of literature establishing that (stochastic) gradient descent converges to a global minimum at exponential rates \citep{dugradient, du2019gradient, allen2019convergence, zou2020gradient, oymak2020toward, liu2022loss, liu2023aiming}. Technically, for the mean squared error (MSE) loss, a sufficiently wide neural network trained with a constant step size $\eta$ achieves a convergence rate of the form $O(\exp(-O(\eta \lambda_{\min} t)))$ where $\lambda_{\min}$ denotes the smallest eigenvalue of the NTK matrix $K$ associated with the corresponding infinitely wide network.

\begin{figure}
    \centering
    \includegraphics[width=0.6\linewidth]{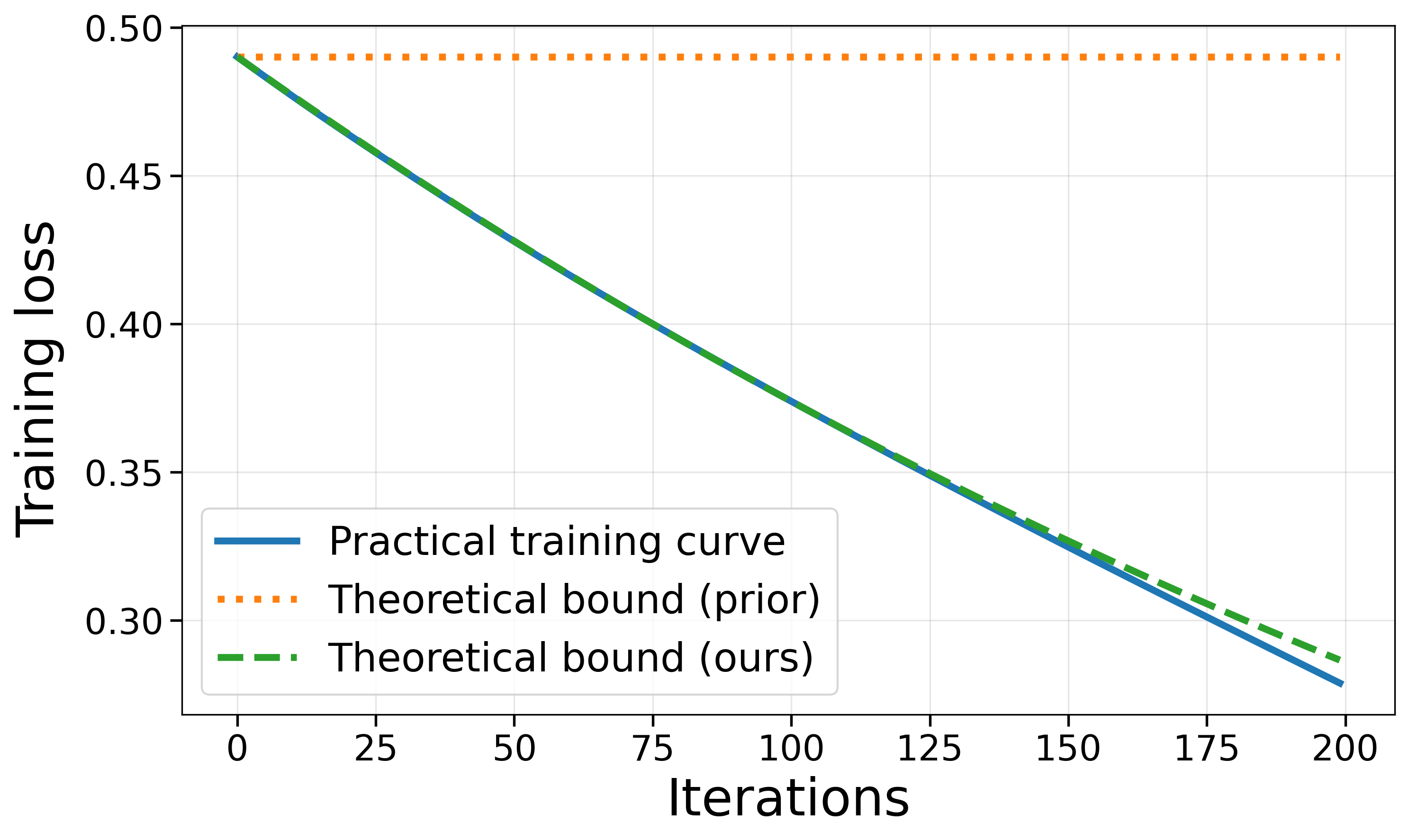}
    \caption{Comparison of theoretical convergence bounds.
      {\bf Solid line}: the actual training curve of a five-layer CNN model trained on a subset of 
    CIFAR-10 \cite{Krizhevsky09learningmultiple} with $1000$ samples under gradient descent. {\bf Orange dotted line}: Theoretical convergence bound $\exp(-\eta\lambda_{\min}t)\mathcal{L}(\rvw_0)$ by prior theories,  slow convergence due to the extremely small value of $\lambda_{\min}$. {\bf Green dashed line}: Theoretical convergence bound $\mathrm{tr}[(I-\eta K)^{2t}K]$ by our theory, fast convergence, and matching the actual training curves.}
    \label{fig:training_curve_intro}
\end{figure}
Despite their elegance and insight in demystifying aspects of deep learning, these optimization theories still fall short of explaining practical behavior.
In particular, their predicted convergence behavior is often overly pessimistic and fails to match empirical observations. This discrepancy manifests in two main aspects. First, the theoretical rates depend critically on the smallest eigenvalue $\lambda_{\min}$, which is typically extremely small in realistic settings, leading to convergence guarantees that are far slower than those observed in practice. For example, as demonstrated in Figure~\ref{fig:training_curve_intro}, its theoretical curve remains nearly flat due to the tiny value of $\lambda_{\min}$ (which is typically even smaller for larger datasets), whereas the empirical training curve decreases much more rapidly. Second, in the NTK regime, features remain essentially fixed during training; as a result, these theories do not capture the rich feature learning dynamics widely believed to drive the success of deep networks. Notably, phenomena such as the edge of stability (EoS) \cite{cohen2021gradient} and catapult dynamics \cite{lewkowycz2020large}, which often arise when using large learning rates, exhibit nonlinear and transient behaviors that lie beyond the predictive scope of the NTK framework.

In this paper, we focus on the NTK regime and aim to derive tighter convergence guarantees. Several prior works have attempted to do this by analyzing the full spectrum of the NTK (rather than only the smallest eigenvalue) under idealized data settings. In particular, assuming the data are uniformly distributed on the unit sphere $\mathbb{S}^d$, a series of works \cite{basri2019convergence,geifman2020similarity,murray2022characterizing} has shown that the eigenvalues of the NTK decay according to a power law and has found that the eigenvectors are spherical harmonic functions, which admit explicit mathematical expressions. Based on this, \cite{basri2019convergence} improved the worst-case convergence rate for a two-layer network with a uniform data distribution on $\mathbb{S}^d$. Additionally, \cite{arora2019finegrainedanalysisoptimizationgeneralization} expresses the convergence bound in terms of the full spectrum and labels the components along each eigen-direction. With the empirical observation that the top few eigenvalues have larger label components, this indicates that the convergence bound can be strongly influenced by the top eigenvalues.

In this paper, we investigate the following central question:

\smallskip
\smallskip
{\it Is it possible to obtain a convergence bound that matches the practical training in the NTK regime? }
\smallskip
\smallskip

Indeed, we have a positive answer. Our approach is based on a key empirical and theoretical finding regarding the relationships between data labels and the NTK eigen-spectrum. Intuitively, these relationships show that the worst-case scenario, on which classical NTK-based optimization theories rely and from which the dependence on $\lambda_{\min}$ arises, does not occur in practice. 
Instead, we find, for real-world datasets, the label vector $\rvy$ is more strongly aligned with eigenvectors corresponding to larger eigenvalues and exhibits weak (often near-orthogonal) alignment with those associated with smaller eigenvalues.
Quantitatively, we uncover the following relationships, which hold for the majority of the NTK spectrum (excluding the top few eigenvalues):

\begin{equation}\label{eq:alignment_relation_intro}
    (\rvv_i^\top \rvy)^2 \sim \lambda_i^2, ~~\mathrm{and}~~ (\rvv_i^\top \rvr)^2 \sim \lambda_i, ~~~~ \forall i\in [n],
\end{equation}
where $\lambda_i$ and $\rvv_i$ are NTK eigenvalues and eigenvectors, respectively, and $\rvr = \rvf-\rvy$ is the residual between model predictions $\rvf\in\mathbb{R}^n$ and the label vector $\rvy$. 
We refer to them as {\it Label-NTK alignment} and {\it Residual-NTK alignment}, respectively.
This is empirically demonstrated in Figure \ref{fig:cifar10_cnn_projections_vs_eigenspectrum}, which shows linear dependence between $\log (\rvv_i^\top \rvy)^2$ and $\log \lambda_i$, and between $\log (\rvv_i^\top \rvr)^2$ and $\log \lambda_i$. Note that the linear relations have a slope $2$ and a slope $1$, respectively, consistent with Eq. (\ref{eq:alignment_relation_intro}). Moreover, we observe that these alignments persist throughout the entire training process. We further provide a theoretical justification of the Label-NTK alignment under mild assumptions satisfied by real-world datasets. Notably, our theory precisely predicts the exponent $2$ in the Label-NTK alignment. Intuitively, this phenomenon arises from the principle that similar inputs tend to have similar labels, which can be mathematically formalized via Lipschitz continuity of the ground-truth data-label function.

\begin{figure}[t]
    \centering
    \begin{subfigure}[t]{0.49\linewidth}
        \centering
        \includegraphics[width=\linewidth]{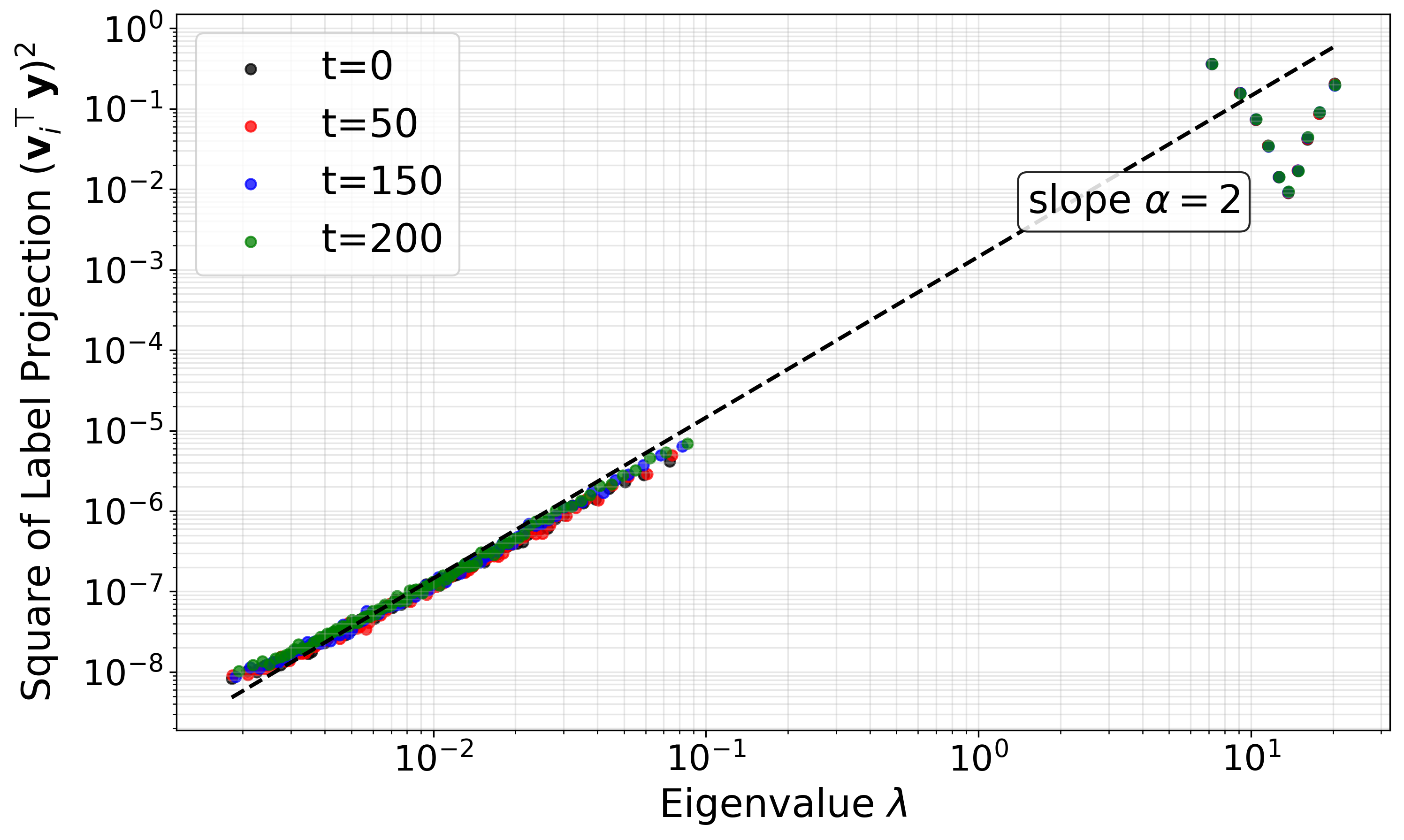}
        \caption{Label-NTK alignment}
        \label{fig:cifar10_y_projections_vs_eigenspectrum_cnn}
    \end{subfigure}
    \hfill
    \begin{subfigure}[t]{0.49\linewidth}
        \centering
        \includegraphics[width=\linewidth]{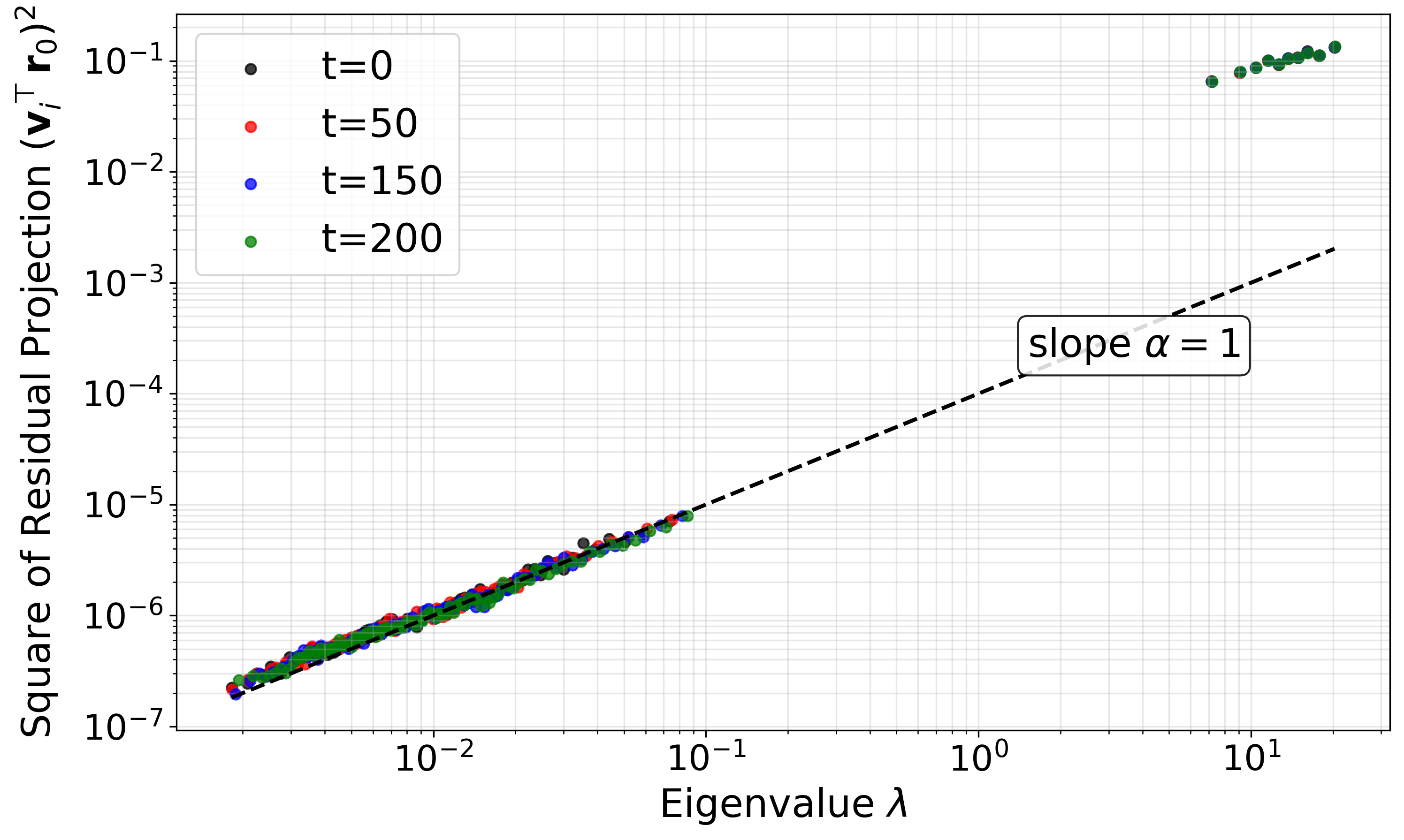}
        \caption{Residual-NTK alignment}
\label{fig:cifar10_r0_projections_vs_eigenspectrum_cnn}
    \end{subfigure}
    \caption{{\it Label-NTK alignment} and {\it Residual-NTK alignment}, of a five-layer CNN (width: $32,64,128,128,256$) trained on the whole CIFAR-10 dataset at different training time $t$. Each point is computed as an average over 500 batches, each containing 100 samples. Note: slope $\alpha=2$ (left) and $\alpha=1$ (right).} \label{fig:cifar10_cnn_projections_vs_eigenspectrum}
\end{figure}

These NTK alignment properties enable us to circumvent explicit dependence on the smallest eigenvalue $\lambda_{\min}$. Leveraging this insight, we establish the following sharper convergence result in the NTK regime:
\begin{theorem}[Informal version of Theorem \ref{thm:loss_bounds}]
Assume the Residual-NTK alignment and other mild data assumptions hold. 
For sufficiently wide networks, gradient descent with constant step size $\eta$ on the squared loss $\mathcal{L}(\rvw_t)$ satisfies:
\begin{equation}
    \mathcal{L}(\rvw_t) \sim \mathrm{tr}[(I-\eta K)^{2t}K].
\end{equation}
\end{theorem}
As demonstrated in Figure~\ref{fig:training_curve_intro}, this bound provides a significantly tighter characterization than prior results and closely matches empirical training curves.

In addition, leveraging the Label-NTK alignment, we derive improved generalization bounds for over-parameterized neural networks compared to prior work.

\noindent{\bf Our contributions.} We summarize our main contributions as follows:\begin{itemize}[leftmargin=*]
\item We find close correlations between the target (i.e., label vector $\rvy$ and residual $\rvr=\rvf-\rvy$) and the NTK eigen-system, particularly in the small-eigenvalue regime. We also provide a theoretical justification for this finding with mild data assumptions.
\item We establish a tighter convergence theory in the NTK regime by incorporating the Residual-NTK alignment. The new convergence bound closely matches the empirical training curves.
\item We provide extensive empirical validation across deep MLP and CNN architectures on multiple datasets, demonstrating strong agreement between theory and observed training dynamics.
\item We derive improved generalization bounds for over-parameterized neural networks based on Label-NTK alignment.
\end{itemize}

\subsection{Related works}

\textbf{Optimization Theories in the NTK regime:}
Neural network optimization has been widely studied using the NTK and related linearization frameworks \cite{jacot2018ntk,lee2019wide,liu2020linearity,liu2022transition,zhu2022transition}. In the infinite width limit, gradient-based training can be characterized as kernel gradient descent with a deterministic kernel fixed at initialization \cite{jacot2018ntk,lee2019wide,chizat2019lazy,geiger2020disentangling}. This view has led to a broad class of convergence results for sufficiently over-parameterized neural networks \cite{dugradient,du2019gradient,allen2019convergence,zou2019improved,zou2020gradient,oymak2020toward,nguyen2020global,ji2020polylogarithmic}. Most such analyses rely on the smallest positive eigenvalue $\lambda_{\min}$ of the NTK in deriving their worse-case convergence rates \cite{jacot2018ntk,dugradient,du2019gradient,lee2019wide}.
While foundational, these bounds are often too loose to precisely describe the optimization behavior exhibited in practice due to their dependence on $\lambda_{\min}$. 

A line of work attempted to improve classical NTK theory by moving beyond $\lambda_{\min}$ and incorporating full-spectral information. Under idealized spherical data models, \cite{basri2019convergence,geifman2020similarity,murray2022characterizing} analyzed the full NTK spectrum, showing structured eigenvalue decay and spherical-harmonic eigenfunctions. Using this structure, \cite{basri2019convergence} derived faster convergence guarantees for two-layer networks. In a different but related direction, \cite{arora2019finegrainedanalysisoptimizationgeneralization} expressed optimization bounds for over-parameterized two-layer networks in terms of the full spectrum and label components along eigen-directions.  \cite{bordelon2020spectrum,canatar2021spectral} emphasized that learning depends on full-kernel spectrum and target distribution across eigen-modes.

\textbf{NTK Alignment:} 
Prior work has noted a connection between the label vector $\rvy$ and the NTK, albeit in a limited sense.  \cite{arora2019finegrainedanalysisoptimizationgeneralization} empirically observed that $\rvy$ exhibit stronger alignment with the leading eigenvalues than the rest, but did not investigate the alignment among the non-leading eigenvalues and did not provide a quantitative characterization of this relationship. There is a line of work that investigates the behavior of an alignment quantity, $\rvy^\top K\rvy$, or its normalized variants, over training time, especially in the cases of small initialization scale or large learning rates where the learning dynamics are beyond the NTK regime \cite{atanasov2022silent,Shan2021ATO,jiang2026understanding,geiger2020disentangling}. 
With small initialization of the network, \cite{atanasov2022silent} observed two phases of the learning dynamics, with the first phase having an increasing alignment $\frac{\rvy^\top K \rvy}{\|\rvy\|_2^2 \, \|K\|_F}$, while \cite{Shan2021ATO} theoretically explained this alignment on shallow networks. 
A recent work \cite{jiang2026understanding} tracks the individual eigenvector-target alignment during training and showed that this alignment also increases over training time and shifts toward leading eigen-directions under larger learning rates. While these works investigate the alignment change in training time,  
our Label-NTK alignment differs in that it describes the difference of the alignment across different eigen-components at each training step.

\section{Background}
\label{preliminaries_notation}
Given a neural network $f: \mathcal{X}\subseteq\mathbb{R}^d \mapsto \mathbb{R}$ with parameters $\rvw\in\mathbb{R}^p$ and a training dataset $\mathcal{D}=\{(\rvx_i,y_i)\}_{i=1}^n$  where $ \rvx_i\in\mathcal{X},\; y_i\in\mathbb{R}$, we consider the optimization problem of minimizing the square loss function $\mathcal{L}(\rvw)$:

\begin{equation}\label{eq:opt_pro}
    \rvw^*  = \arg\min_{\rvw}\mathcal{L}(\rvw), ~~~ \mathrm{with}~~~\mathcal{L}(\rvw) =  \frac{1}{2}\sum_{i=1}^n\big(f_\rvw(\rvx_i)-y_i\big)^2 \triangleq \frac{1}{2}\|\rvf(\rvw) - \rvy\|^2,
\end{equation}
where we use $\rvf$ and $\rvy$ to denote the vectors of model outputs and labels on the dataset: $\rvf(\rvw) \triangleq (f_\rvw(\rvx_1), f_\rvw(\rvx_2), \cdots, f_\rvw(\rvx_n))^\top$ and $\rvy \triangleq (y_1, y_2, \cdots, y_n)^\top$. To emphasize the dependence of the network output on the parameters $\rvw$, we also denote $f_\rvw(\rvx)$ as $f(\rvw;\rvx)$.

We optimize the loss function using the gradient descent (GD) algorithm, which updates the network parameters $\rvw$ following.
$\rvw_{t+1}=\rvw_t-\eta\nabla_\rvw \mathcal{L}(\rvw_t)$, where $t$ demotes the time step, $\eta>0$ is the learning rate. 

In this paper, we focus on over-parameterized neural networks, where the number of network parameters $p$ is much greater than the training dataset size $n$: $p\gg n$. In the over-parameterized settings, neural networks can often exactly fit the training data, a.k.a. interpolation, and achieve zero training loss: there exists $\rvw^*$ such that $\mathcal{L}(\rvw^*)=0$. The over-parameterization is often achieved by letting the network width $m$, i.e., the number of neurons in each hidden layer, be sufficiently large. Indeed, most existing optimization theories of neural networks consider the sufficiently large width setting, where a nice mathematical property -- constant NTK -- has been found and is set as the base of these theories.  

\textbf{Neural Tangent Kernel (NTK).} Given two arbitrary inputs $\rvx_1$ and $\rvx_2$, the NTK $\mathcal{K}$ is defined as 
\begin{equation}
    \mathcal{K}(\rvw;\rvx_1,\rvx_2) = \langle \nabla_\rvw f(\rvw;\rvx_1), \nabla_\rvw f(\rvw;\rvx_2) \rangle,
\end{equation}
where $\nabla_\rvw f$ denotes the derivative of network output w.r.t. parameters $\rvw$. Given the dataset $\mathcal{D}$, there is a Jacobian matrix $J(\rvw)\in\mathbb{R}^{n\times p}$, each row $J_i(\rvw)$ of which is the vector $\nabla_\rvw f_\rvw(\rvx_i)$, and an NTK matrix $K(\rvw)\in\mathbb{R}^{n\times n}$ such that
\begin{equation*}
    K(\rvw) = J(\rvw)J(\rvw)^\top, ~~ \mathrm{and}~~ K_{ij}(\rvw) = \mathcal{K}(\rvw;\rvx_i,\rvx_j).
\end{equation*}

The key theoretical finding is that, in the infinite width limit, the NTK matrix remains unchanged during the network training process, $K(\rvw_t) = K$ for all $t\ge 0$ \cite{jacot2018ntk}, or even in a large region, such as a ball, in the parameter space that contains the training path \cite{liu2022loss}.
Based on this property, previous work shows that the training dynamics of a neural network with finite but sufficiently large width $m$ is approximately linear \cite{lee2019wide,liu2020linearity}:
\begin{equation}\label{eq:f-y_existing}
    \rvf(\rvw_t) - \rvy \approx (I-\eta K)^t (\rvf(\rvw_0)-\rvy),
\end{equation}
and then the loss function can be controlled by an exponentially decreasing sequence; hence, it enjoys a fast convergence rate under gradient descent \cite{allen2019convergence,du2019gradient}. For example:  
\begin{proposition}[Informal version of Theorem 8 of \cite{liu2022loss}]\label{thm:prior_opt}
Let $\lambda_{\min}$ be the smallest eigenvalue of the infinite-width NTK $K$. If the network width satisfies $m=\Omega(\lambda_{\min}^{-2})$, then with a constant step size $\eta = O(1)$, gradient descent converges to a global minimizer with the following rate:
\begin{equation}\label{eq:slow_rate}
    \mathcal{L}(\rvw_t)\le (1-\eta\lambda_{\min})^{t}\mathcal{L}(\rvw_0).
\end{equation}
\end{proposition}

As demonstrated in Figure \ref{fig:training_curve_intro}, the theoretical convergence bound depicted by Eq.~(\ref{eq:slow_rate}), although exponential in principle, is typically quite flat and cannot match the empirically observed fast convergence. In this paper, we aim to provide a tighter theoretical convergence bound so that it matches the empirical training curves.

\section{Label-NTK Alignment and Residual-NTK alignment}
\label{sec:label-ntk-cor}
In this section, we present our findings of a strong correlation between the NTK matrix $K$ and the data label vector $\rvy$, and provide a theoretical justification for this phenomenon. This correlation helps address the limitations of existing neural network theories and facilitates improved analyses of optimization and generalization, which will be discussed in Section \ref{improved_optimization_and_generalization}.

Our {\bf key observation} regarding the limitations of existing theories (namely, their reliance on the extreme value of $\lambda_{\min}$) is that these analyses do not account for the relationship between the NTK matrix $K$ and the label vector $\rvy$, or the residual vector $\rvf(\rvw_0)-\rvy$. As a result, in order to derive worst-case upper bounds, they must implicitly assume the worst-case scenario in which $\rvy$ lies in the eigen-space corresponding to the eigenvalue $\lambda_{\min}$. Consequently, even with the nice relation in Eq.~(\ref{eq:f-y_existing}), the resulting convergence rate is still impractically slow, and the required network width becomes excessively large. This raises the key questions: \emph{Does this worst-case scenario actually occur in practice?} Or, more importantly, \emph{is there any correlation between these quantities that prevents such a worst case from arising?}

\subsection{Label-NTK alignment} 
Here, we show both experimentally and theoretically that the label vector $\rvy$ is strongly correlated with the NTK matrix $K$. In particular, we argue that $\rvy$ has much smaller, often negligible, components in the eigen-spaces associated with small eigenvalues. Technically, we observe that the following relation holds  for some $\alpha > 0$:
\begin{equation}\label{eq:y-ntk-align}
    (\rvv_i^\top\rvy)^2 \sim \lambda_i^\alpha,
\end{equation}
where $\lambda_i$ and $\rvv_i$ denote the eigenvalues and eigenvectors of the NTK matrix $K$, with the eigenvalues $\lambda_i$ ordered in descending order.

Our experiments demonstrate that, across different network architectures (CNN in Figure \ref{fig:cifar10_y_projections_vs_eigenspectrum_cnn} and MLP in Figure \ref{fig:cifar10_y_projections_vs_eigenspectrum_mlp}), $\log (\rvv_i^\top\rvy)^2$ and $\log \lambda_i$ always have a linear relationship, especially at the small-eigenvalue regime, obeying the scaling law of Eq. (\ref{eq:y-ntk-align}). 
It is also interesting to observe that the scaling power $\alpha$ has a value around $2$. Additional plots on more datasets can be found in Appendix \ref{sec_app:additional_alignment_plots}. 

Moreover, we find that this Label-NTK alignment is preserved during training, even though the NTK matrix may change due to the finite network width.

\begin{figure}[t]
    \centering
    \begin{subfigure}[t]{0.49\linewidth}
        \centering
        \includegraphics[width=\linewidth]{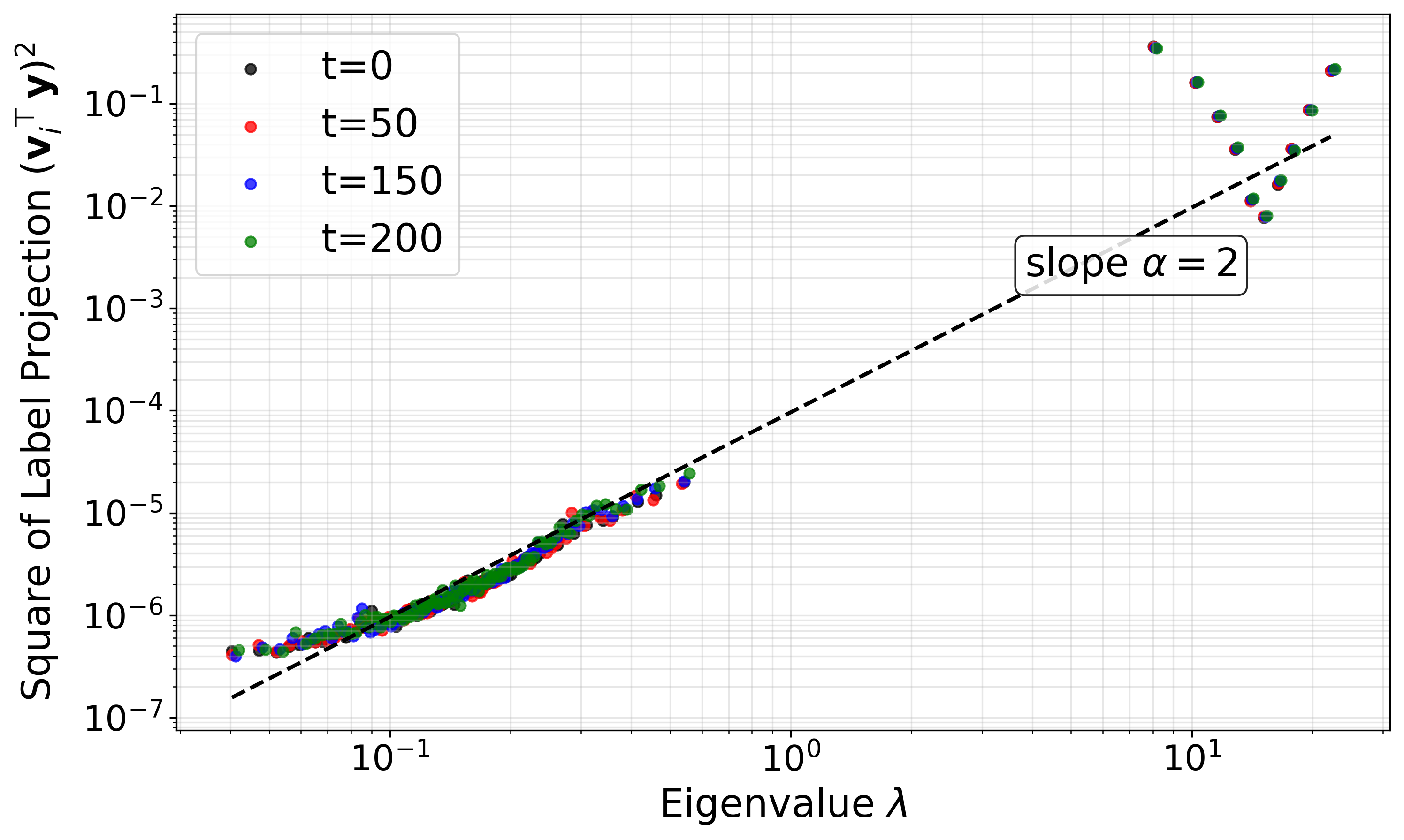}
        \caption{Label-NTK alignment}
        \label{fig:cifar10_y_projections_vs_eigenspectrum_mlp}
    \end{subfigure}
    \hfill
    \begin{subfigure}[t]{0.49\linewidth}
        \centering
        \includegraphics[width=\linewidth]{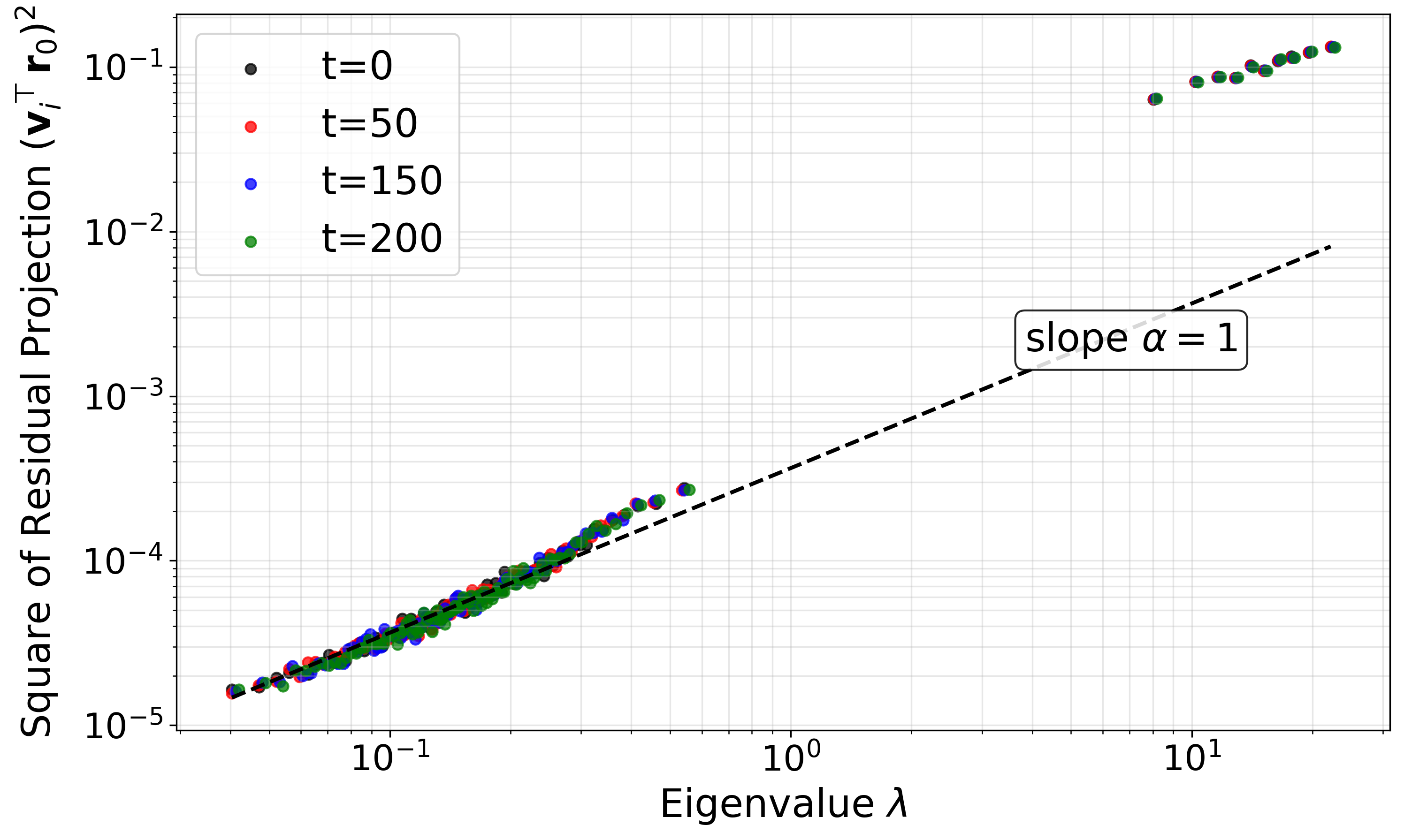}
        \caption{Residual-NTK alignment}
\label{fig:cifar10_r0_projections_vs_eigenspectrum_mlp}
    \end{subfigure}
    \caption{{\it Label-NTK alignment} and {\it Residual-NTK alignment}, of a five-layer MLP (width$=512$) trained on the whole CIFAR-10 dataset at different training time $t$. Each point is computed as an average over 500 batches, each containing 100 samples. Note: slope $\alpha=2$ (left) and $\alpha=1$ (right).} \label{fig:cifar10_mlp_projections_vs_eigenspectrum}
\end{figure}

\paragraph{Theoretical justification.} Now, we theoretically justify this label-NTK correlation with a focus on small eigenvalues. We make the following assumptions about the dataset.
\begin{assumption}[Data assumption]
\label{assum:lipschitz_continuity}
In the dataset $\mathcal{D}=\{(\rvx_i, y_i)\}_{i=1}^n$, 
\begin{itemize}[leftmargin=*]
    \item ({\bf Unit input length}). inputs $\|\rvx_i\|=1$ for all $i\in [n]$,
    \item ({\bf Non-degeneracy}). $\rvx_i \nparallel \rvx_j$ for all $i\ne j \in [n]$,
    \item ({\bf Lipschitz continuity of the ground truth function}). There exists a constant $L_y>0$ such that $|y_i-y_j|\le L_y\|\rvx_i-\rvx_j\|$ hold for all $i,j\in [n]$.
\end{itemize}
\end{assumption}
The first two data assumptions, unit length and non-degeneracy, are standard simplifications in theoretical analyses, see for example \cite{du2019gradient,dugradient}. The assumption that the ground-truth data function is Lipschitz continuous is also natural, as it formalizes the intuition that {\it similar inputs have similar labels}: if two inputs are very close $\|\rvx_i-\rvx_j\|\ll 1$, then their labels $y_i$ and $y_j$ should not be significantly different, which is often the case in real-world data. In classification settings, where the $y_i$ represent discrete categories, this condition corresponds to the existence of a margin between classes, such as assumed in \cite{ji2020polylogarithmic}.

\noindent {\it Dataset with $n=2$ data points.} 
Let's start with a simple case to illustrate how the label-NTK alignment arises. We consider a dataset $\mathcal{D}_2=\{(\rvx_1,y_1), (\rvx_2, y_2)\}$ of size $2$, where the inputs $\rvx_1$ and $\rvx_2$ are very close: $\|\rvx_1-\rvx_2\| \ll 1$.   Denote the angle between the two inputs as $\theta$. Under Assumption \ref{assum:lipschitz_continuity}, we have $0<\theta\ll1$ and $|y_1-y_2|=O(\theta)$. To simplify the analysis, we consider the NTK matrix $K$ in the infinite width limit. For networks of finite but sufficiently large width, the empirical NTK deviates from $K$ only by an asymptotically negligible amount \cite{jacot2018ntk,du2019gradient,dugradient,liu2020linearity}.

\begin{theorem}[label-NTK Alignment for two-point dataset]
\label{thm:two_points_small}
Consider the dataset $\mathcal{D}_2$ with $0<\theta\ll1$ and a fully-connected neural network of depth $L$ and infinite width. Suppose the data Assumption \ref{assum:lipschitz_continuity} holds. Then, 
\begin{equation}
    (\rvv_{\min}^\top \rvy)^2 \sim \lambda_{\min}^2 = O(\theta^2), ~~ (\rvv_{\max}^\top \rvy)^2 \sim \lambda_{\max}^2 = O(1).
\end{equation}
\end{theorem}

\begin{remark}
It is worth pointing out that the above theorem not only explains the strong correlation between the label vector $\rvy$ and the NTK eigen-system, but also precisely predicts the scaling exponent $\alpha = 2$, which we observed in Figures \ref{fig:cifar10_y_projections_vs_eigenspectrum_cnn} and \ref{fig:cifar10_y_projections_vs_eigenspectrum_mlp}.
\end{remark}
\begin{proof}
As $\nabla_\rvw f(\rvw; \rvx)$ is independent of $\rvw$ for the infinitely wide neural network and the context is clear that $\nabla$ is the derivative w.r.t. $\rvw$, we denote $\nabla_\rvw f(\rvw; \rvx)$ simply as $\nabla f(\rvx)$.
With the unit input length of data inputs, Assumption \ref{assum:lipschitz_continuity}, it has been found in \cite{liu2025better} that $\frac{1}{L+1}\|\nabla f(\rvx_i)\| = \|\rvx_i\| = 1$ for all $i$.
In the case of two data points, the NTK $K$ is a $2\times 2$ matrix and has the following form.
\[
\bar{K} \triangleq \frac{1}{(L+1)^2} K
=
\begin{pmatrix}
1 & \cos\phi\\
\cos\phi & 1
\end{pmatrix},
\]
where $\phi$ is the angle between the two vectors $\nabla f(\rvx_1)$ and  $\nabla f(\rvx_2)$.
It is easy to see that $\bar{K}$ has the following eigenvalues and eigenvectors:
\[
\lambda_{\max} = 1+\cos \phi, ~~ \rvv_{\max} = \tfrac{1}{\sqrt2}(1,1)^\top,\quad \lambda_{\min} = 1-\cos \phi, ~~ \rvv_{\min} = \tfrac{1}{\sqrt2}(1,-1)^\top.
\]
It has been shown by previous work, Corollary~3.3 of ~\cite{liu2025better}, that the angle $\phi$ has the following relation with the angle between the two inputs $\theta$:
\begin{equation}
\cos\phi=\big(1-a\theta+o(\theta)\big)\cos\theta,
\qquad a:=\frac{L}{2\pi}.
\label{eq:phi_theta_cor_in_thm}
\end{equation}
Hence, we have
\begin{equation}
   \lambda_{\max} \ge  1, ~~~ \lambda_{\min} = O(\theta).
\end{equation}

Using the formula for 
$\rvv_{\min}$ and $\rvv_{\max}$, we get
\[
(\rvv_{\min}^\top \rvy)^2=\frac{(y_1-y_2)^2}{2}, \quad (\rvv_{\max}^\top \rvy)^2=\frac{(y_1+y_2)^2}{2} = O(1).
\]
Because the true data function $y(\rvx)$ is Lipschitz continuous, by Assumption~\ref{assum:lipschitz_continuity}, we get
\begin{align*}
    (\rvv_{\min}^\top \rvy)^2
\le \frac{L_{y}^2}{2}\|\rvx_1-\rvx_2\|_2^2
= L_{y}^2(1-\cos\theta)
= O(\theta^2).
\end{align*}
\end{proof}

\noindent {\it Dataset with $n>2$ data points.} We now consider the general case with $n > 2$ data points and show that the above result extends. We assume that there exists at least one pair of distinct data points, $\rvx_i$ and $\rvx_j$ with $i \ne j$, that are very close to each other; in this case, the NTK will exhibit eigenvalues of small magnitude. 
This scenario commonly arises in real-world datasets, particularly when the dataset size $n$ is large.

\begin{theorem}
\label{thm:n_points_pair_projection_spectrum}
Consider the dataset $\mathcal{D}$ of size $n>2$ and a fully-connected neural network of depth $L$ and infinite width. Suppose the data Assumption \ref{assum:lipschitz_continuity} holds and there is a pair of data points $\rvx_i$ and $\rvx_j$, $i\ne j$, such that the angle $\theta$ between the two data points is small $\theta\ll1$. Then, there is at least one eigenvalue $\lambda$ and its corresponding eigenvector $\rvv$ satisfy the following:
\begin{equation}
    (\rvv^\top \rvy)^2 \sim \lambda^2 = O(\theta^2).
\end{equation}
\end{theorem}
The proof of the theorem is deferred to Appendix \ref{sec_app:pf_multi_data}.
\begin{remark}
The small eigenvalue $\lambda$ and the corresponding small eigen-projection $(\rvv^\top \rvy)^2$ arise from the small angular separation between the pair of data points $\rvx_i$ and $\rvx_j$. If more such closely spaced pairs exist, we expect additional small eigenvalues and, correspondingly, more small eigen-projections.
\end{remark}

\subsection{Residual-NTK alignment} 
Define the initial residual vector as $\rvr_0 = \rvf(\rvw_0)-\rvy$, the difference between the label vector and the network outputs at initialization. We also observe a strong correlation between the residual $\rvr_0$ with $\rvy$, which is similar to the label-NTK alignment. As shown in Figures \ref{fig:cifar10_r0_projections_vs_eigenspectrum_cnn} and \ref{fig:cifar10_r0_projections_vs_eigenspectrum_mlp}, except for the top eigenvalues, the log of projection $(\rvv_i^\top \rvy)^2$ also linearly depends on the log of eigenvalues $\lambda_i$.  Namely,
\begin{equation}
    (\rvv_i^\top \rvr_0)^2 \sim \lambda_i^\alpha.
    \label{eq:r0_y_correlation}
\end{equation}
Interestingly, the scaling power $\alpha$ of Residual-NTK alignment is observed to be about $1$, instead of the value $2$ for Label-NTK alignment. Theoretically, it is not a surprise that the residual $\rvr_0$ has a similar correlation with NTK, because intuitively the neural network function $f$ at initialization should have a certain level of Lipschitz continuity w.r.t. the input $\rvx$, which is similar to that of the true function $y(\rvx)$ as we required in Assumption \ref{assum:lipschitz_continuity}. Additional plots can be found in Appendix \ref{sec_app:additional_alignment_plots}.

\section{Improved Optimization Theory for Wide Neural Networks}
\label{improved_optimization_and_generalization}
In this section, we develop an optimization theory for wide neural networks that provides a tighter convergence bound, utilizing the correlations observed from the previous section.

Throughout this section, we assume the loss function $\mathcal{L}(\rvw)$ is $\beta$-smooth and the network function $f(\rvw)$ is $\beta_f$-smooth with $\beta,\beta_f\ge 1$, where the smoothness is defined below:
\begin{definition}
    We say a function $g$ is $\beta$-smooth, if 
\begin{equation*}
    \|\nabla g(\rvw_1)-\nabla g(\rvw_2)\| \le \beta \|\rvw_1-\rvw_2\|, ~~~ \forall \rvw_1, \rvw_2.
\end{equation*}
\end{definition}
Without loss of generality, we assume the loss at initialization $\mathcal{L}(\rvw_0) \ge 1$.

Based on the findings of the strong residual-NTK alignment  at initialization, we put the following condition:
\begin{condition}[Residual-NTK alignment]
\label{assum:residual_ntk_alignment}
There exist constants \(c\gg\delta>0\) such that
\begin{equation}\label{eq:res_ntk_alignment}
    (c-\delta)\lambda_i \le (\rvv_i^\top \rvr_0)^2 \le (c+\delta)\lambda_i,
\quad \forall i\in[n].
\end{equation}
\end{condition}
\begin{remark}
In this condition, we only require that the relation holds at initialization, although we empirically observed in Section \ref{sec:label-ntk-cor} that the same relation still holds during and after training.
\end{remark}
\begin{remark}
Compared to Eq. (\ref{eq:r0_y_correlation}), this condition permits fluctuations around the expectation $(\rvv_i^\top \rvr_0)^2 \sim \lambda_i$, thereby making it milder.
\end{remark}

With all the above setup, we are ready to present the main theorem. 
\begin{theorem}\label{thm:loss_bounds}
Consider a fully-connected neural network $f$ with $L$ hidden layers. Suppose Condition \ref{assum:residual_ntk_alignment} holds.  For all $\varepsilon \in (0,1)$, if the network width satisfies $m=\tilde{\Omega}\left(\frac{nR^{6L+2}}{\lambda_{\min}^2\varepsilon^2}\right)$ with $R=\frac{8\beta\sqrt{\mathcal{L}(\rvw_0)}}{\lambda_{\min}}$ and the learning rate $\eta\le \frac{1}{2\beta+\beta_f\sqrt{2\mathcal{L}(\rvw_0)}}$, then the square loss $\mathcal{L}(\rvw)$ under gradient descent has the following convergence bounds:
\begin{equation}\label{eq:thm_loss_bounds}
    \frac{c-\delta}{2}\texttt{tr}[(I-\eta K)^{2t}K] - \frac{\varepsilon}{2}
    \le \mathcal{L}(\rvw_t) \le \frac{c+\delta}{2}\texttt{tr}[(I-\eta K)^{2t}K] + \frac{\varepsilon}{2}. 
\end{equation}
\end{theorem}
\begin{proof}
    The proof is deferred to Appendix \ref{sec_app:pf_loss_bounds}.
\end{proof}

The upper bound established in this theorem depends on the trace of a function of the NTK matrix $K$, rather than on the smallest eigenvalue $\lambda_{min}$, resulting in a substantially tighter bound than those in prior work. We numerically validate this bound for a convolutional neural network in Figure \ref{fig:training_curve_intro}. Additional demonstrations across different datasets and network architecture (i.e., MLP) are provided in Appendix \ref{sec_app:additional_alignment_plots}. Across all settings, we observe that our convergence bound is significantly tighter than existing bounds and tracks the true training curves much more accurately.

Moreover, the above theorem also provides a lower bound on the convergence, which is of the same order as the upper bound, suggesting the optimality of the upper bound.

The following corollary explicitly writes out the dependence of the convergence bound on the loss value $\mathcal{L}(\rvw_0)$ at initialization. 
\begin{corollary}
    In the same setting as Theorem \ref{thm:loss_bounds}, the loss function during training is upper bounded by:
    \begin{equation}\label{eq:bound_with_l_0}
        \mathcal{L}(\rvw_t) \le \frac{c+\delta}{c-\delta}\frac{\texttt{tr}[(I-\eta K)^{2t}K]}{\texttt{tr}[K]}\mathcal{L}(\rvw_0) +\frac{\varepsilon}{2}.
    \end{equation}
\end{corollary}
\begin{proof}
Consider the left inequality of Eq. (\ref{eq:res_ntk_alignment}) and sum over $i$, we get
\begin{equation*}
    (c-\delta)\texttt{tr}[K] \le 2\mathcal{L}(\rvw_0).
\end{equation*}
Combining with the right inequality of Eq. (\ref{eq:thm_loss_bounds}), we obtain the result. 
\end{proof}
Note that when $t=0$, the upper bound in Eq. (\ref{eq:bound_with_l_0}) reduces to $\frac{c+\delta}{c-\delta}\mathcal{L}(\rvw_0) +\frac{\varepsilon}{2}$. Given that $\delta \ll c$ and $\varepsilon \ll 1$, we have this upper bound tight.

\paragraph{Time Complexity bound.}
Now, we consider the time complexity to achieve a certain loss value $\varepsilon>0$. We additionally assume the following power law decay in NTK eigenvalues.
\begin{assumption}[Eigenvalue power law decay]
    \label{prop:eigenvalue_power_law_decay}
NTK matrix $K \in \mathbb{R}^{n \times n}$ is positive definite, and there exist constants $C_1,C_2>0$ and $q>1$ such that the following holds:
    \begin{equation}
        C_1i^{-q}\leq \lambda_i \leq C_2i^{-q}, \quad \forall i \in [n].
    \end{equation}
\end{assumption} 
This power law decay of NTK eigenvalues has been discovered by a large body of prior work \cite{basri2019convergence,basri2020frequency,geifman2020similarity,velikanov2021explicit,murray2022characterizing} in various settings, including the special setting of uniformly distributed data in the unit sphere $\mathbb{S}^d$.

\begin{theorem}\label{thm:time_complexity}
For all $\varepsilon\in(0,1)$, consider the same setting as Theorem \ref{thm:loss_bounds}. A training loss $\mathcal{L}(\rvw_T)\le \varepsilon$ can be achieved within $T_\varepsilon$ iterations where
\begin{equation}
    T_\varepsilon = \frac{1}{2\eta C_1} \left(\frac{(q-1)\varepsilon}{2(c+\delta)C_2} + n^{1-q}\right)^{\frac{q}{1-q}}
    \log\left(\frac{4(c+\delta)\mathcal{L}_0}{(c-\delta)\varepsilon}\right).
\end{equation}
\end{theorem}
\begin{proof}
Proof is deferred to Appendix \ref{sec_app:pf_time_complexity}.
\end{proof}
The above time complexity bound is of the following order:
\begin{equation}
    T_\varepsilon = \min \left\{O\left(\frac{1}{\eta \varepsilon^{1+\frac{1}{q-1}}} \log \left(\frac{1}{\varepsilon}\right)\right), O\left(\frac{n^q}{\eta} \log \left(\frac{1}{\varepsilon}\right)\right)\right\}.
\end{equation}
For comparison, the standard $\lambda_{\min}$-dependence convergence bound $O(\exp(-\eta\lambda_{\min}t))$ induces a time complexity of $O(\frac{1}{\eta\lambda_{\min}}\log \frac{1}{\varepsilon})$. As $\varepsilon$ is typically much larger than $\lambda_{\min}$ in practice (for example, one often sets $\varepsilon=10^{-4}$ or larger), the complexity obtained in the above theorem is tighter.

\section{Improving Generalization Theory}
\label{sec:improving_generalization_theory}
We show that the Label-NTK alignment also helps to improve the generalization bound for an over-parameterized neural network. Our analysis is based on the following established generalization bound:

\begin{proposition}[Informal version of Theorem 5.1 of \cite{arora2019finegrainedanalysisoptimizationgeneralization}]\label{prop:generalization}
Let network initialization scale $\kappa = O(\lambda_{\min}/n)$ and network width $m=\Omega(\kappa^{-2} \mathrm{poly}(\lambda_{\min}^{-1}))$. The test loss $\mathcal{L}_{test}$ of a two-layer ReLU network trained by gradient descent is upper bounded by 
\begin{equation}\label{eq:generalization_existing}
    \mathcal{L}_{test} \le \sqrt{2\rvy^\top K^{-1}\rvy/n}.
\end{equation}
\end{proposition}

\begin{theorem}
Consider the same setting as in Theorem 5.1 of \cite{arora2019finegrainedanalysisoptimizationgeneralization}.  Suppose the Label-NTK alignment holds: $(\rvv_i^\top\rvy)^2 \sim \lambda_i^2$. Then, the test loss $\mathcal{L}_{test}$ of a two-layer ReLU network trained by gradient descent is upper bounded by 
\begin{equation}
    \mathcal{L}_{test} \le O(\mathrm{tr}[K]/n).
\end{equation}
\end{theorem}
\begin{proof}
    \begin{equation}
        \rvy^\top K^{-1}\rvy = \sum_{i=1}^n \lambda_i^{-1}(\rvv_i^\top\rvy)^2 = \sum_{i=1}^n \lambda_i^{-1} (c\cdot\lambda_i^2) = c\sum_{i=1}^n\lambda_i = c\cdot\mathrm{tr}[K],
    \end{equation}
    where $c$ is a certain constant. The result directly follows.
\end{proof}

\section{Conclusion and Limitations}
\label{sec:conclusion_and_limitations}
In this paper, we experimentally observed, and theoretically justified, strong correlations between the training label vector $\rvy$ and the NTK eigen-components: Label-NTK alignment $(\rvv_i^\top\rvy)^2 \sim \lambda_i^2$ and Residual-NTK alignment $(\rvv_i^\top\rvy)^2 \sim \lambda_i$. These alignments show that, even at initialization, the data label has less (or ignorable) correlation with  NTK eigen-components of lower eigenvalues. This finding enabled us to develop a convergence bound that is much tighter than previous results in the NTK regime and explains the fast convergence observed in practice.  

\noindent{\bf Limitation.} The theory of this paper is limited to the NTK regime, where the neural network is sufficiently wide, and the learning rate is relatively low so that feature learning dynamics, such as Edge of Stability, do not show up. Despite the limitations of the NTK regime, we argue that it is still worth developing theory in this regime: The feature learning regime is still an open area lacking well-established convergence analyses, whereas the NTK regime represents one of the most mature and mathematically tractable frameworks available. As such, it continues to offer valuable insight into the behavior of modern over-parameterized neural networks and serves as a foundation for future theoretical advances beyond its assumptions.

\section*{Acknowledgement}
This research used both the DeltaAI advanced computing and data resource, which is supported by the National Science Foundation (award OAC 2320345) and the State of Illinois, and the Delta advanced computing and data resource which is supported by the National Science Foundation (award OAC 2005572) and the State of Illinois. Delta and DeltaAI are joint efforts of the University of Illinois Urbana-Champaign and its National Center for Supercomputing Applications.

\bibliographystyle{plain}
\bibliography{references}

\newpage
\appendix

\section{Proofs}
\label{sec_app:omitted_proofs}
\subsection{Proof of Theorem \ref{thm:n_points_pair_projection_spectrum}}
\label{sec_app:pf_multi_data}
    Without loss of generality, we assume that $i=1$ and $j=2$; namely, the angle $\theta$ between the pair $\rvx_1$ and $\rvx_2$ is small $\theta\ll 1$. 

    We consider the normalized NTK matrix $\bar{K} = \frac{1}{(L+1)^2}K$. As has been shown in \cite{liu2025better}, $\frac{1}{(L+1)^2}\|\nabla f(\rvx_i)\| = \|\rvx_i\| = 1$ for all $\rvx_i\in\mathcal{D}$. Therefore, all the diagonal entries of $\bar{K}$ are $1$, and each off-diagonal entry $\bar{K}_{ij}$, $i\ne j$, is $\cos \phi_{ij}$ with $\phi_{ij}$ being the angle between the two tangent vectors $\nabla f(\rvx_i)$ and $\nabla f(\rvx_j)$.

We write the normalized NTK matrix $\bar{K}$ in the following form:
\begin{equation}
    \bar{K}=
\begin{pmatrix}
1 & \cos \phi_{12} & \rva^\top\\
\cos \phi_{12} & 1 & \rvb^\top\\
\rva & \rvb & B
\end{pmatrix},
\end{equation}
where vectors $\rva,\rvb\in \mathbb{R}^{n-2}$ and matrix $B\in\mathbb{R}^{(n-2)\times (n-2)}$.  By Theorem 3.1 of \cite{dugradient}, the matrix $B$ is positive-definite and invertible. 

By Corollary of~\cite{liu2025better}, we get 
\begin{equation*}
    \cos \phi_{12} = (1-a\theta + o(\theta))\cos \theta, ~~ \textrm{with~} a = \frac{L}{2\pi}.
\end{equation*}
In addition, considering components of $\rva$ and $\rvb$, we have, for each $i\in [n-2]$,
\begin{equation*}
    |a_i-b_i| = |\cos \phi_{1(i+2)}-\cos \phi_{2(i+2)}| = O(\theta_{1(i+2})-\theta_{2(i+2)}) = O(\theta).
\end{equation*}
In brief, we can write 
\begin{equation}
    \bar{K}=
\begin{pmatrix}
1 & 1-a\theta +o(\theta) & \rva^\top\\
1-a\theta +o(\theta) & 1 & \rva^\top+ O(\theta)\\
\rva & \rva + O(\theta) & B
\end{pmatrix}.
\end{equation}
We can see that it has the following eigenvalue $\lambda$ and corresponding eigenvector $\rvv$:
\begin{equation*}
    \lambda = a\theta +o(\theta), ~~~ \rvv = \frac{1}{\sqrt{2}}(1, ~ -1, ~ -B^{-1}(\rva-\rvb)) + o(\theta).
\end{equation*}
Hence, 
\begin{equation*}
    \rvv^\top \rvy = \frac{1}{\sqrt{2}}(y_1-y_2 - \tilde{\rvy}^\top B^{-1}(\rva-\rvb)) + o(\theta).
\end{equation*}
Note that $\rva-\rvb = O(\theta)$. Additionally, by Assumption ~\ref{assum:lipschitz_continuity}, we also have 
$|y_1-y_2| \le L_y\|\rvx_1-\rvx_2\| = O(\theta)$. Therefore, we obtain
\begin{equation}
    (\rvv^\top \rvy)^2 = O(\theta^2), ~~ \textrm{and~} \lambda = O(\theta). 
\end{equation}
We complete the proof. 

\subsection{Proof of Theorem \ref{thm:loss_bounds}}\label{sec_app:pf_loss_bounds}
We start the proof with the following key lemma (Proof of lemma is deferred to Appendix \ref{sec_app:pf_lemma_residual_t}).
\begin{lemma}
\label{lem:residual_at_time_t}
Given any $\varepsilon \in (0,1)$, if the model Hessian spectral norm $\|\nabla^2f\|_{op} \le \frac{\lambda_{\min}\varepsilon}{32 \beta R \sqrt{n}\mathcal{L}(\rvw_0)}$ holds within the ball $B_R(\rvw_0)$ with $R = \frac{8 \sqrt{\beta\mathcal{L}(\rvw_0)}}{\lambda_{\min}}$ and learning rate $\eta\le \frac{1}{2\beta+\beta_f\sqrt{2\mathcal{L}(\rvw_0)}}$, then 
\begin{equation}\label{eq:linear_dynamic}
    \rvf(\rvw_t)-\rvy = (I-\eta K)^t(\rvf(\rvw_0)-\rvy) + \Delta_t,
\end{equation}
with $\|\Delta_t\| \le \frac{\varepsilon}{4\sqrt{\mathcal{L}(\rvw_0)}}$.
\end{lemma}
Prior work, Theorem 3.2 of \cite{liu2020linearity}, has proven that the model Hessian norm can be arbitrarily small given a large enough $m$: $\|\nabla^2f\|_{op} = \tilde{O}\left(\frac{R^{3L}}{\sqrt{m}}\right)$. Therefore, the conditions of the lemma are met by the neural network set in the theorem. 

Taking square on both sides of Eq. (\ref{eq:linear_dynamic}), we get
\begin{align*}
    \mathcal{L}(\rvw_t) &= \frac{1}{2}\|\rvf(\rvw_t)-\rvy\|^2 = \frac{1}{2}\|(I-\eta K)^t\rvr_0\|^2 + ((I-\eta K)^t\rvr_0)^\top \Delta_t + \|\Delta_t\|^2.
\end{align*}

On the other hand, multiplying a factor $(1-\eta \lambda_i)^{2t}$  onto both sizes of Eq. (\ref{eq:res_ntk_alignment}) in Assumption \ref{assum:residual_ntk_alignment}
and summing over $i$ gives
\begin{equation*}
    (c-\delta)\sum_{i=1}^n (1-\eta \lambda_i)^{2t}\lambda_i
\le
\sum_{i=1}^n (1-\eta \lambda_i)^{2t} (\rvv_i^\top \rvr_0)^2
\le
(c+\delta)\sum_{i=1}^n (1-\eta \lambda_i)^{2t}\lambda_i
\end{equation*}

Note that $\|(I-\eta K)^t\rvr_0\|^2=\sum_{i=1}^n (1-\eta \lambda_i)^{2t} (\rvv_i^\top \rvr_0)^2$ and $\sum_{i=1}^n (1-\eta \lambda_i)^{2t}\lambda_i = \texttt{tr}[(I-\eta K)^{2t}K]$, then we obtain
\begin{align*}
    (c-\delta) \texttt{tr}[(I-\eta K)^{2t}K]\le \|(I-\eta K)^t\rvr_0\|^2 \le (c+\delta) \texttt{tr}[(I-\eta K)^{2t}K].
\end{align*}

In addition, we have 
\begin{align*}
|((I-\eta K)^t\rvr_0)^\top \Delta_t| &\le \|(I-\eta K)^t\rvr_0\| \cdot \|\Delta_t\| \le \|\rvr_0\|\cdot \|\Delta_t\|\le \frac{\varepsilon}{4},
\end{align*}
and $\|\Delta_t\|^2 \le \frac{\varepsilon^2}{16\mathcal{L}(\rvw_0)}\le \frac{\varepsilon}{4}$.
Combining all the above, we finally get the desired result.

\subsection{Proof of Lemma \ref{lem:residual_at_time_t}}
\label{sec_app:pf_lemma_residual_t}
Using Taylor expansion with Lagrange remainder term, as well as the GD update rule $\rvw_{t+1}-\rvw_t = -\eta J(\rvw_t)^\top(\rvf(\rvw_t)-\rvy)$, we have the following:
\begin{align}
\rvf(\rvw_{t+1})-\rvf(\rvw_{t}) &= -\eta K(\rvw_t) (\rvf(\rvw_t)-\rvy) + \boldsymbol{\zeta}_t \nonumber\\
&= -\eta K(\rvw_0) (\rvf(\rvw_t)-\rvy) - \eta\delta K_t (\rvf(\rvw_t)-\rvy) + \boldsymbol{\zeta}_t \label{eq-append:f_t+1-f_t}
\end{align}
where $\delta K_t = K(\rvw_t)-K(\rvw_0)$ and vector $\boldsymbol{\zeta}_t\in\mathbb{R}^n$ has elements
\begin{equation}\label{eq-append:zeta}
    \zeta_{t,i} = \eta^2 (\rvf(\rvw_t)-\rvy)^\top J(\rvw_t)\nabla^2 f(\xi_{t}(\rvx_i); \rvx_i)J(\rvw_t)^\top(\rvf(\rvw_t)-\rvy), ~~ \forall i\in[n],
\end{equation}
where $\xi(\rvx_i)$ is some point on the line segment between $\rvw_t$ and $\rvw_{t+1}$ depending on $\rvx_i$.

Note that, under the setting of the model Hessian spectral norm $\|\nabla^2f\|_{op} \le \frac{\lambda_{\min}\varepsilon}{32 \beta R \sqrt{n}\mathcal{L}(\rvw_0)}$,  $R = \frac{8 \sqrt{\beta\mathcal{L}(\rvw_0)}}{\lambda_{\min}}$ and learning rate $\eta\le \frac{1}{2\beta+\beta_f\sqrt{2\mathcal{L}(\rvw_0)}}$, \cite{liu2022loss} has proven that:
\begin{itemize}
    \item $\rvw_t$ remains in the ball $B_R(\rvw_0)$ for all $t\ge 0$.
    \item $\mathcal{L}(\rvw_t) = \frac{1}{2}\|\rvf(\rvw_t)-\rvy\|^2 \le \left(1-\eta \frac{\lambda_{\min}}{2}\right)^t\frac{1}{2}\|\rvf(\rvw_0)-\rvy\|^2$.
\end{itemize}

Now, let's find the upper bounds for the magnitudes of the two terms: $\eta\delta K_t (\rvf(\rvw_t)-\rvy)$ and $\boldsymbol{\zeta}_t$.

{\it Bounding term $\eta\delta K_t (\rvf(\rvw_t)-\rvy)$.} For all $i\in [n]$, we have
\begin{equation*}
    \nabla_\rvw f(\rvw_t; \rvx_i)= \nabla_\rvw f(\rvw_0; \rvx_i) + \int_0^1 \nabla^2 f(\rvw_0+\tau(\rvw_t-\rvw_0; \rvx_i)(\rvw_t-\rvw_0)d\tau.
\end{equation*}
Then we have
\begin{align*}
\|\nabla_\rvw f(\rvw_t; \rvx_i)- \nabla_\rvw f(\rvw_0; \rvx_i)\| &\le \sup_{\tau\in[0,1]} \|\nabla^2 f(\rvw_0+\tau(\rvw_t-\rvw_0;\rvx_i))\|_{op} \cdot \|\rvw_t-\rvw_0\|\\& \le \frac{\lambda_{\min}\varepsilon}{32 \beta R \sqrt{n}\mathcal{L}(\rvw_0)} \cdot R \\& = \frac{\lambda_{\min}\varepsilon}{32\beta \sqrt{n}\mathcal{L}(\rvw_0)}.
\end{align*}

Hence, 
\begin{align*}
\|J(\rvw_t)-J(\rvw_0)\|_{op} &\le \|J(\rvw_t)-J(\rvw_0)\|_{F} \\
& = \sqrt{\sum_{i\in[n]} \|\nabla_\rvw f(\rvw_t; \rvx_i)- \nabla_\rvw f(\rvw_0; \rvx_i)\|^2}\\
&\le \frac{\lambda_{\min}\varepsilon}{32\beta {\mathcal{L}(\rvw_0)}}.
\end{align*}

Now, the spectral norm of the NTK change $\delta K_t$ is bounded by
\begin{align*}
\|\delta K_t\|_{op} &= \|J(\rvw_t)J(\rvw_t)^\top-J(\rvw_0)J(\rvw_0)^\top\|_{op} \\
& = \|J(\rvw_t)(J(\rvw_t)-J(\rvw_0))^\top + (J(\rvw_t)-J(\rvw_0))J(\rvw_0)^\top\|_{op}\\
&\le (\|J(\rvw_t)\|_{op} + \|J(\rvw_0)\|_{op})\|J(\rvw_t)-J(\rvw_0)\|_{op} \\
&\le  \frac{\lambda_{\min}\varepsilon}{8 {\mathcal{L}(\rvw_0)}}.
\end{align*}
In the last inequality above, we use the fact that $\|J(\rvw)\|_{op}^2 = \|K(\rvw)\|_{op}$.

Therefore, the term $\eta\delta K_t (\rvf(\rvw_t)-\rvy)$ is bounded by
\begin{align*}
    \|\eta\delta K_t (\rvf(\rvw_t)-\rvy)\| \le \frac{\eta\lambda_{\min}\varepsilon}{8 {\mathcal{L}(\rvw_0)}}\|\rvf(\rvw_t)-\rvy\|.
\end{align*}

{\it Bounding term $\boldsymbol{\zeta}_t$.} 
By Eq. (\ref{eq-append:zeta}), we have, for all $t>0, i\in[n]$,
\begin{align*}
|\zeta_{t,i}| &\le \eta^2 \|\rvf(\rvw_t)-\rvy\|^2 \|K(\rvw_t)\|_{op} \cdot \|\nabla^2 f(\xi_t(\rvx_i);\rvx_i)\|_{op}\\
&\le \eta^2 \frac{\lambda_{\min}\varepsilon}{16 R \sqrt{n}\mathcal{L}(\rvw_0)} \|\rvf(\rvw_t)-\rvy\|^2.
\end{align*}
Then, we have
\begin{align*}
   \|\boldsymbol{\zeta}_{t}\| = \sqrt{\sum_{i\in[n]}|\zeta_{t,i}|^2} \le \eta^2 \frac{\lambda_{\min}\varepsilon}{16 R {\mathcal{L}(\rvw_0)}} \|\rvf(\rvw_t)-\rvy\|^2.
\end{align*}

Finally, applying Eq. (\ref{eq-append:f_t+1-f_t}) recursively, we get
\begin{align*}
\rvf(\rvw_t)-\rvy &= (I-\eta K(\rvw_0))^t(\rvf(\rvw_0)-\rvy) + \sum_{\tau=0}^{t-1}(I-\eta K(\rvw_0))^{t-1-\tau} \left(\eta\delta K_\tau (\rvf(\rvw_\tau)-\rvy) + \boldsymbol{\zeta}_\tau\right)\\
& \triangleq (I-\eta K(\rvw_0))^t(\rvf(\rvw_0)-\rvy) + \Delta_t.
\end{align*}
We can see that the first term $(I-\eta K(\rvw_0))^t(\rvf(\rvw_0)-\rvy)$ is equivalent to a linear dynamics. The deviation term (deviating from linear dynamics) can be bounded by 
\begin{align*}
 \|\Delta_t\| &\le \sum_{\tau=0}^{t-1}\|I-\eta K(\rvw_0)\|_{op}^{t-1-\tau} \left(\|\eta\delta K_\tau (\rvf(\rvw_\tau)-\rvy)\| + \|\boldsymbol{\zeta}_\tau\|\right)\\
 &\le \sum_{\tau=0}^{t-1}(1-\eta \lambda_{\min})^{t-1-\tau}\cdot \frac{\eta\lambda_{\min}\varepsilon}{4{\mathcal{L}(\rvw_0)}}\|\rvf(\rvw_t)-\rvy\| \\
 & \le \frac{\eta\lambda_{\min}\varepsilon}{4\sqrt{\mathcal{L}(\rvw_0)}}\sum_{\tau=0}^{t-1}(1-\eta \lambda_{\min})^{t-1-\tau} \\
 &\le  \frac{\varepsilon}{4\sqrt{\mathcal{L}(\rvw_0)}}.
\end{align*}
Above we used $\|\rvf(\rvw_t)-\rvy\|^2 \le \left(1-\eta \frac{\lambda_{\min}}{2}\right)^t\|\rvf(\rvw_0)-\rvy\|^2 \le \|\rvf(\rvw_0)-\rvy\|^2 = \mathcal{L}(\rvw_0)$ for the third inequality, and $\sum_t q^t \le 1/(1-q)$ for $q<1$ for the fourth inequality.

\subsection{Proof of Theorem \ref{thm:time_complexity}}\label{sec_app:pf_time_complexity}
According to Theorem \ref{thm:loss_bounds}, we have 
\begin{equation}
    \mathcal{L}(\rvw_t) \le \frac{c+\delta}{2}\texttt{tr}[(I-\eta K)^{2t}K] + \frac{\varepsilon}{2} = \frac{c+\delta}{2}\sum_{i=1}^n(1-\eta \lambda_i)^{2t}\lambda_i + \frac{\varepsilon}{2}.
    \label{eq:pf_bound_1}
\end{equation}
By Proposition \ref{prop:eigenvalue_power_law_decay}, we have
$\lambda_i \le C_2 i^{-q}.$ Define 
\begin{equation*}
    i^*_\varepsilon = \left(\frac{(q-1)\varepsilon}{2C_2(c+\delta)} + n^{1-q}\right)^{\frac{1}{1-q}}, 
\end{equation*}
Then we have
\begin{align*}
    \frac{c+\delta}{2}\sum_{i=i^*_\varepsilon+1}^n(1-\eta \lambda_i)^{2t}\lambda_i &\le \frac{c+\delta}{2}\sum_{i=i^*_\varepsilon+1}^n\lambda_i \\
    &\le \frac{c+\delta}{2}\sum_{i=i^*_\varepsilon+1}^nC_2 i^{-q} \\
    &\le \frac{c+\delta}{2}\int_{i^*_\varepsilon}^\infty C_2 z^{-q} dz\\
    &= \frac{(c+\delta)C_2}{2} \frac{(i^*_\varepsilon)^{1-q}}{q-1} \\
    &\le \frac{\varepsilon}{4}.
\end{align*}
Then Eq. (\ref{eq:pf_bound_1}) reduces to 
\begin{equation}
    \mathcal{L}(\rvw_t) \le \frac{c+\delta}{2}\sum_{i=i^*_\varepsilon}^n(1-\eta \lambda_i)^{2t}\lambda_i + \frac{3\varepsilon}{4}.
    \label{eq:pf_bound_2}
\end{equation}
Now consider the term $\sum_{i=i^*_\varepsilon}^n(1-\eta \lambda_i)^{2t}\lambda_i$. Denote $\lambda_{i=i^*_\varepsilon}$ as $\lambda_*$. We have
\begin{align*}
    \sum_{i=i^*_\varepsilon}^n(1-\eta \lambda_i)^{2t}\lambda_i &\le (1-\eta \lambda_*)^{2t}\sum_{i=i^*_\varepsilon}^n \lambda_i \\
    &\le (1-\eta \lambda_*)^{2t}\sum_{i=1}^n \lambda_i \\
    &\le (1-\eta C_1 (i^*_\varepsilon)^{-q})^{2t} \cdot \frac{2\mathcal{L}_0}{c-\delta}.
\end{align*}
In the last inequality above, we used $\lambda_* \ge C_1(i^*_\varepsilon)^{-q}$ (Assumption \ref{prop:eigenvalue_power_law_decay}) and Condition \ref{assum:residual_ntk_alignment}.
When 
\begin{equation*}
    t \ge \frac{1}{2\eta C_1} \left(\frac{(q-1)\varepsilon}{2(c+\delta)C_2} + n^{1-q}\right)^{\frac{q}{1-q}}
    \log\left(\frac{4(c+\delta)\mathcal{L}_0}{(c-\delta)\varepsilon}\right),
\end{equation*}
we can bound $\mathcal{L}(\rvw_t) \le \varepsilon$. Hence, we finish the proof.

\begin{figure}[t]
    \centering
    \begin{subfigure}[t]{0.47\linewidth}
        \centering
        \includegraphics[width=\linewidth]{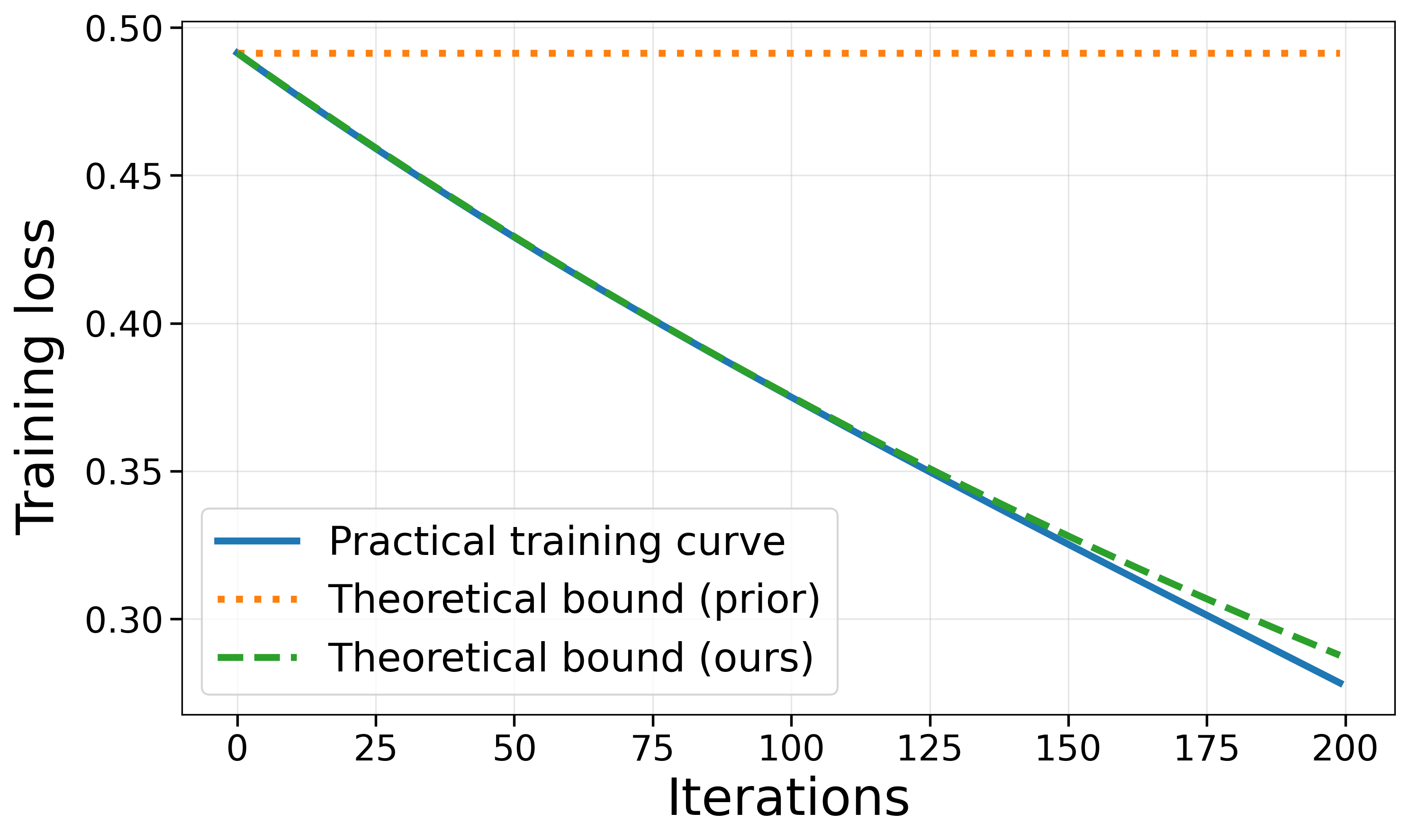}
        \caption{Tiny-ImageNet10, CNN}
        \label{fig:tinyimagenet10_cnn_convergence}
    \end{subfigure}
    \hfill
    \begin{subfigure}[t]{0.47\linewidth}
        \centering
        \includegraphics[width=\linewidth]{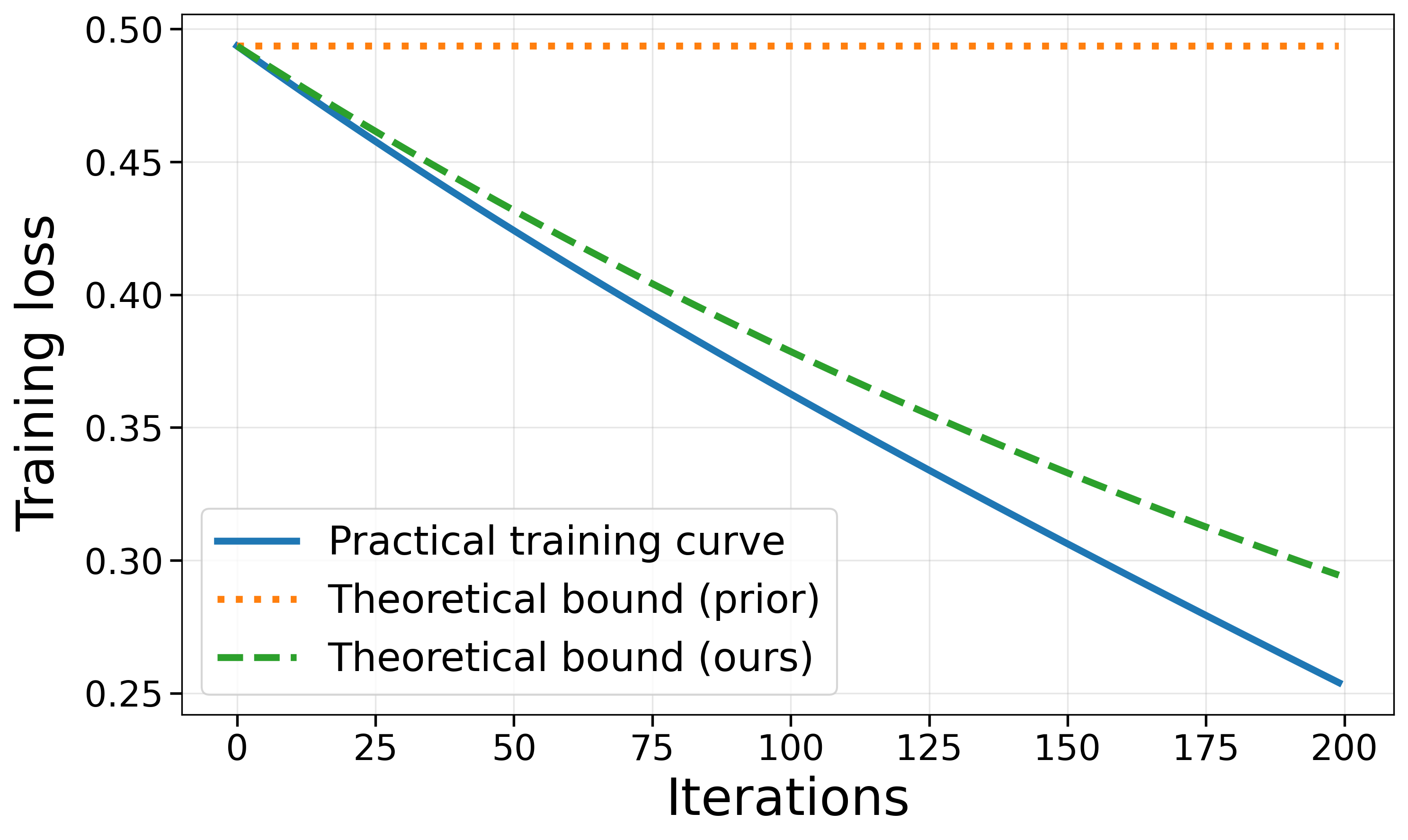}
        \caption{CIFAR-10, MLP}
        \label{fig:cifar10_mlp_convergence}
    \end{subfigure}

    \vspace{0.4em}

    \begin{subfigure}[t]{0.47\linewidth}
        \centering
        \includegraphics[width=\linewidth]{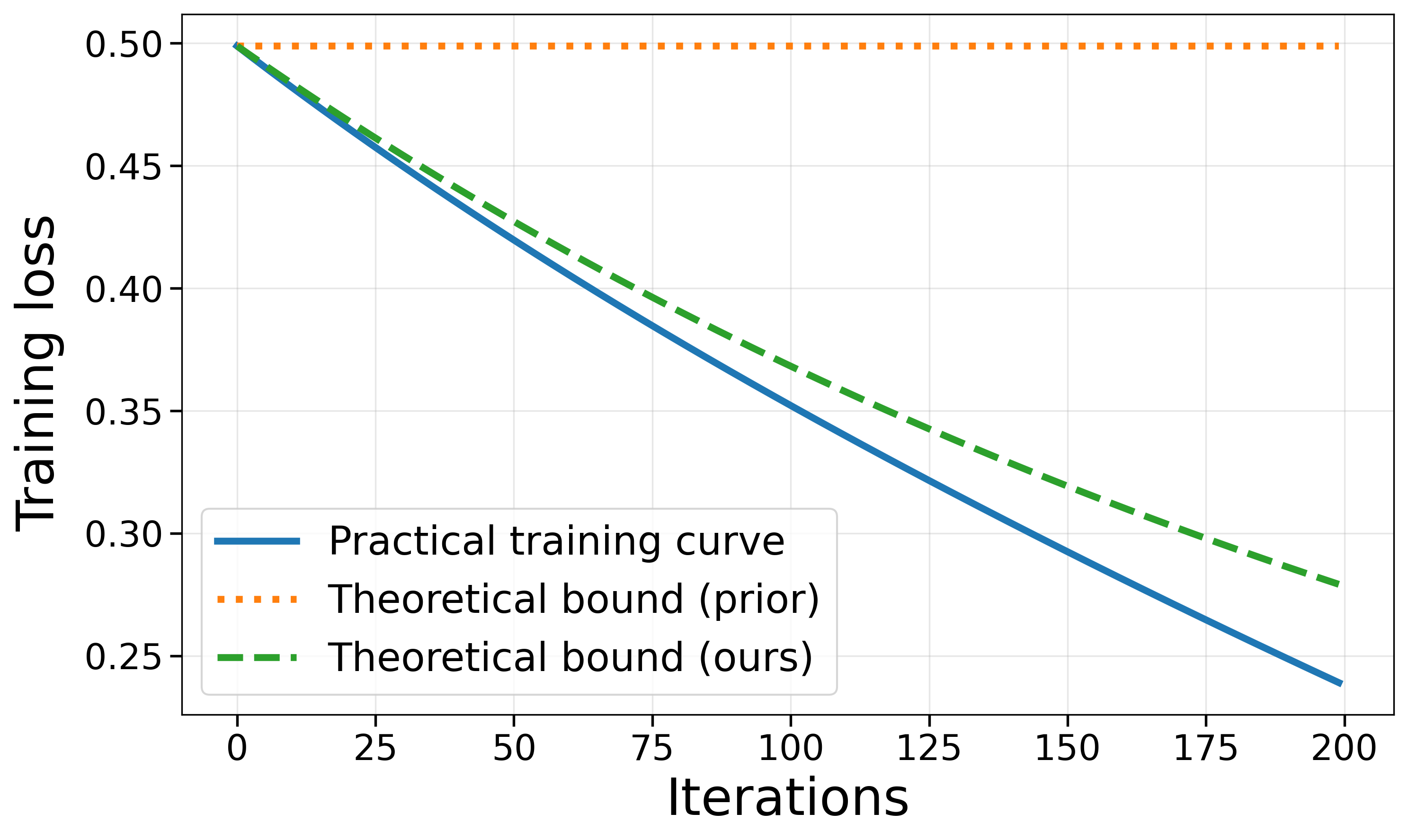}
        \caption{SVHN, MLP}
        \label{fig:svhn_mlp_convergence}
    \end{subfigure}
    \hfill
    \begin{subfigure}[t]{0.47\linewidth}
        \centering
        \includegraphics[width=\linewidth]{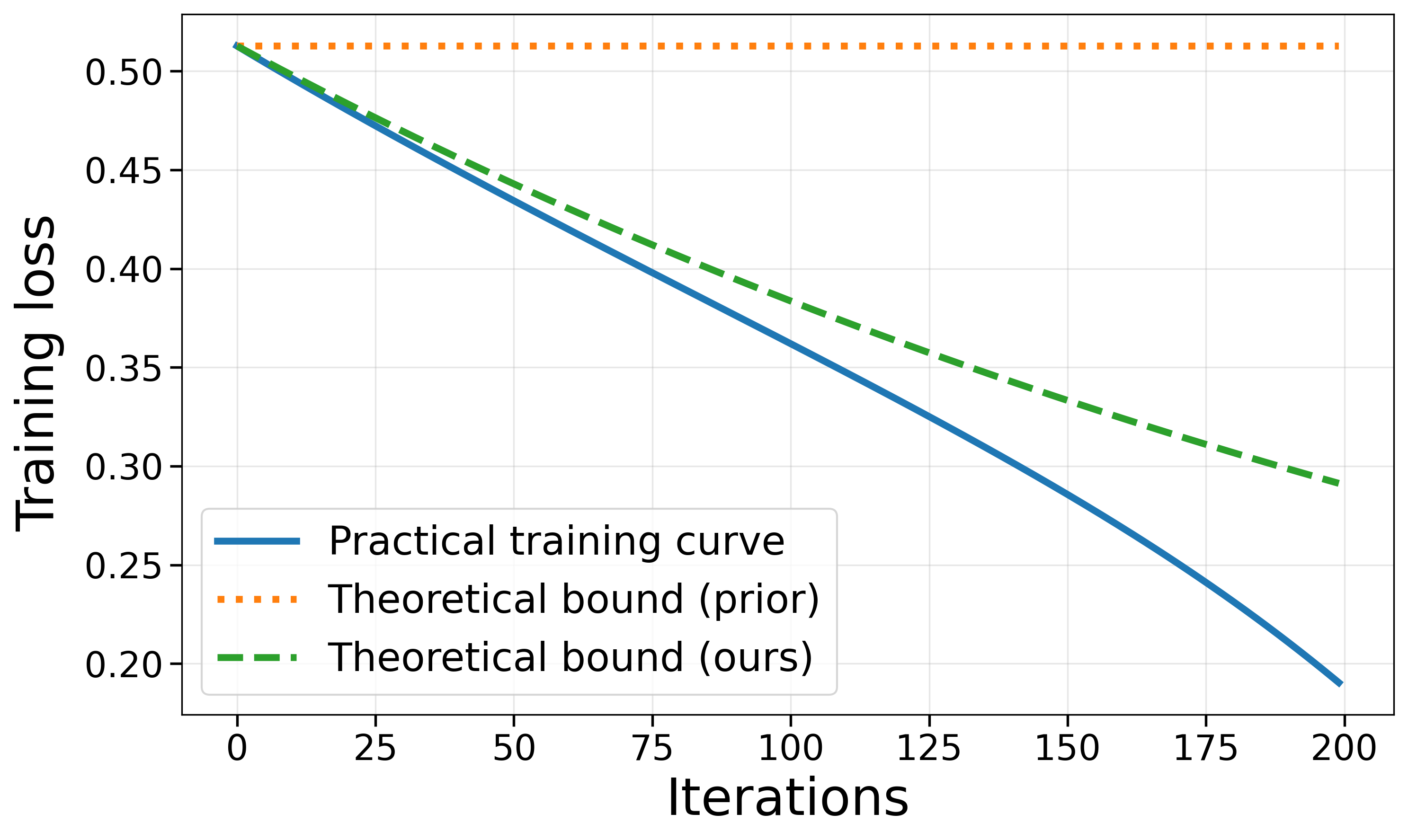}
        \caption{Fashion-MNIST, MLP}
        \label{fig:fashionmnist_mlp_convergence}
    \end{subfigure}

    \vspace{0.4em}

    \begin{subfigure}[t]{0.47\linewidth}
        \centering
        \includegraphics[width=\linewidth]{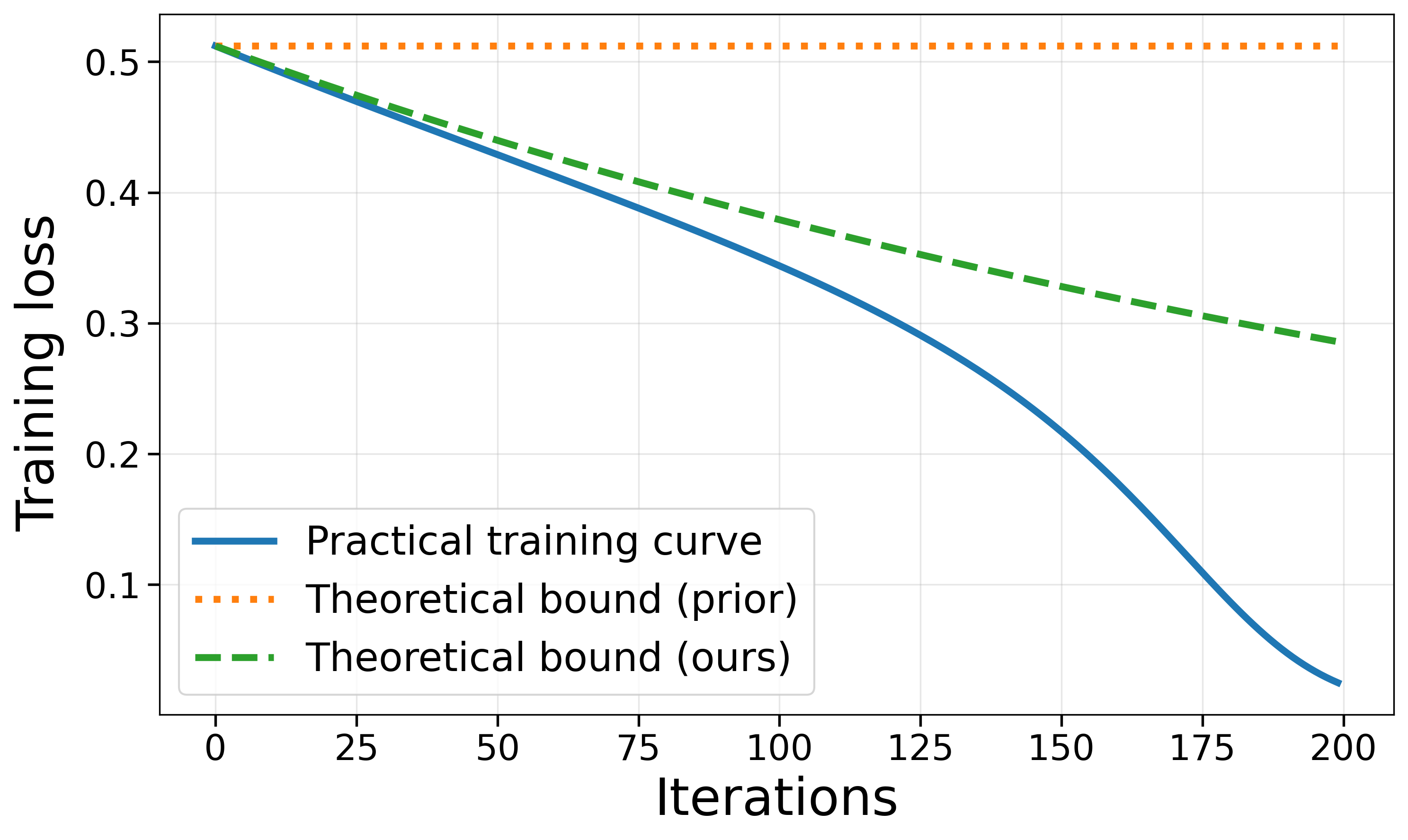}
        \caption{MNIST, MLP}
        \label{fig:mnist_mlp_convergence}
    \end{subfigure}

    \caption{Comparison of theoretical convergence bounds across architectures and datasets. The MLP experiments use a five-layer architecture on subsets of CIFAR-10, SVHN, Fashion-MNIST, and MNIST, while the CNN experiment is evaluated on Tiny-ImageNet10. Each training subset contains $1000$ samples.}
    \label{fig:mlp_curves}
\end{figure}

\section{Experimental Setup and Additional Plots}
\label{sec_app:additional_alignment_plots}

\subsection{Details of experimental setup}
\textbf{Architectures.} The same two finite-width architectures, deep MLP and CNN, were used throughout the NTK Alignment and convergence empirical studies in PyTorch. Deep MLP consists of five hidden layers, each of width 512, with ReLU activations after every hidden layer and a final linear output layer. The CNN consists of five $3 \times 3$ convolutional layers with channel sizes $32,64,128,128,$ and $256$, ReLU activations throughout, and max-pooling after the second, fourth, and fifth convolution blocks. This
is followed by adaptive $4 \times 4$ pooling and a three-layer classifier $4096 \rightarrow 512 \rightarrow 256 \rightarrow 10$. For all the experiments, SGD Optimizer with momentum as $0$, learning rate as $0.01$ are chosen consistently.

\textbf{NTK Alignment.} The numerical results for Label-NTK and Residual-NTK alignments are obtained by computing the eigen-spectrum of the  NTK together with target projections across eigen-modes at multiple time stamps. In particular, Label-dependent projections $(\rvv_i^\top \rvy)^2$ and residual-dependent projections $(\rvv_i^\top \rvr)^2$ are evaluated on batches of $100$ samples to avoid computation precision issue for small eigenvalues and we report their average over all batches over the whole datasets.  The experiments are conducted on MNIST, Fashion-MNIST, SVHN, and CIFAR-10 Datasets for MLP, and CIFAR-10 and Tiny-ImageNet-10  Datasets for CNN. 

\textbf{Convergence.} For the convergence experiments, full-batch GD was run with a learning rate of $0.01$, $200$ iterations, and a subset of $1000$ samples. The experiments are validated on MNIST, Fashion-MNIST, SVHN, and CIFAR-10 Datasets for MLP, and CIFAR-10 and Tiny-ImageNet-10 Datasets for CNN. 
For each setting, we use  the NTK $K$ at initialization, or its smallest eigenvalue $\lambda_{\mathrm{min}}$, to compute the theoretical convergence bounds.  

\subsection{Additional Plot for Convergence Rates}
Figure \ref{fig:mlp_curves} compares the convergence bound obtained from our theory with that of prior theories, in the setting of a five-layer MLP trained on the whole CIFAR-10 dataset. 

\subsection{Additional NTK Alignment Plots}
Figure \ref{fig:alignment_additional_plots} provides more evidence on Label-NTK alignment and Residual-NTK alignment across various architectures and datasets. 


\begin{figure}[t]
    \centering
    \begin{subfigure}[t]{0.45\linewidth}
        \centering
        \includegraphics[width=\linewidth]{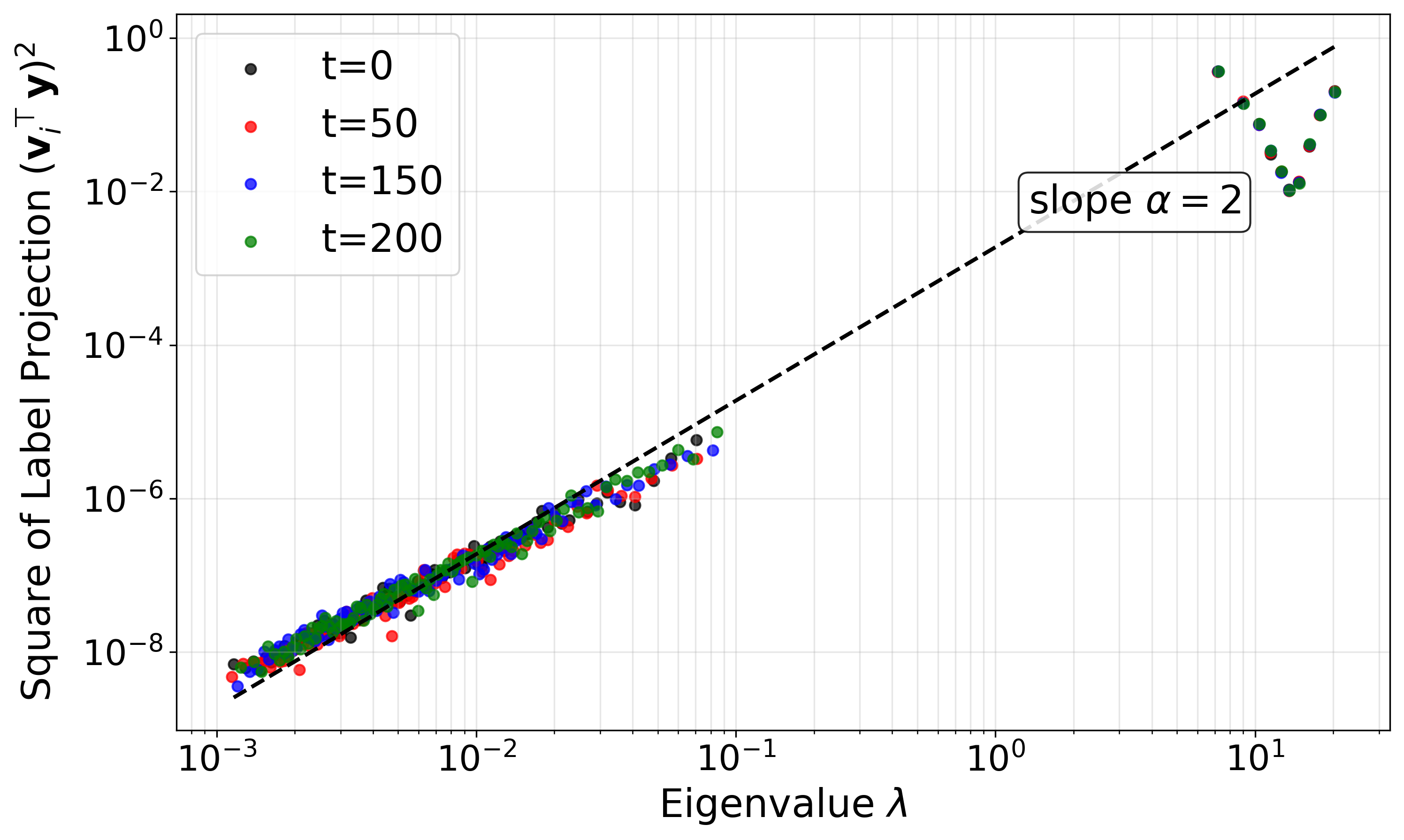}
        \vspace{-5pt}
        \label{fig:cnn_y_projections_tinyimagenet10}
        \caption{Label-NTK alignment(CNN,Tiny-ImageNet10)}
    \end{subfigure}
    \hfill
    \begin{subfigure}[t]{0.45\linewidth}
        \centering
        \includegraphics[width=\linewidth]{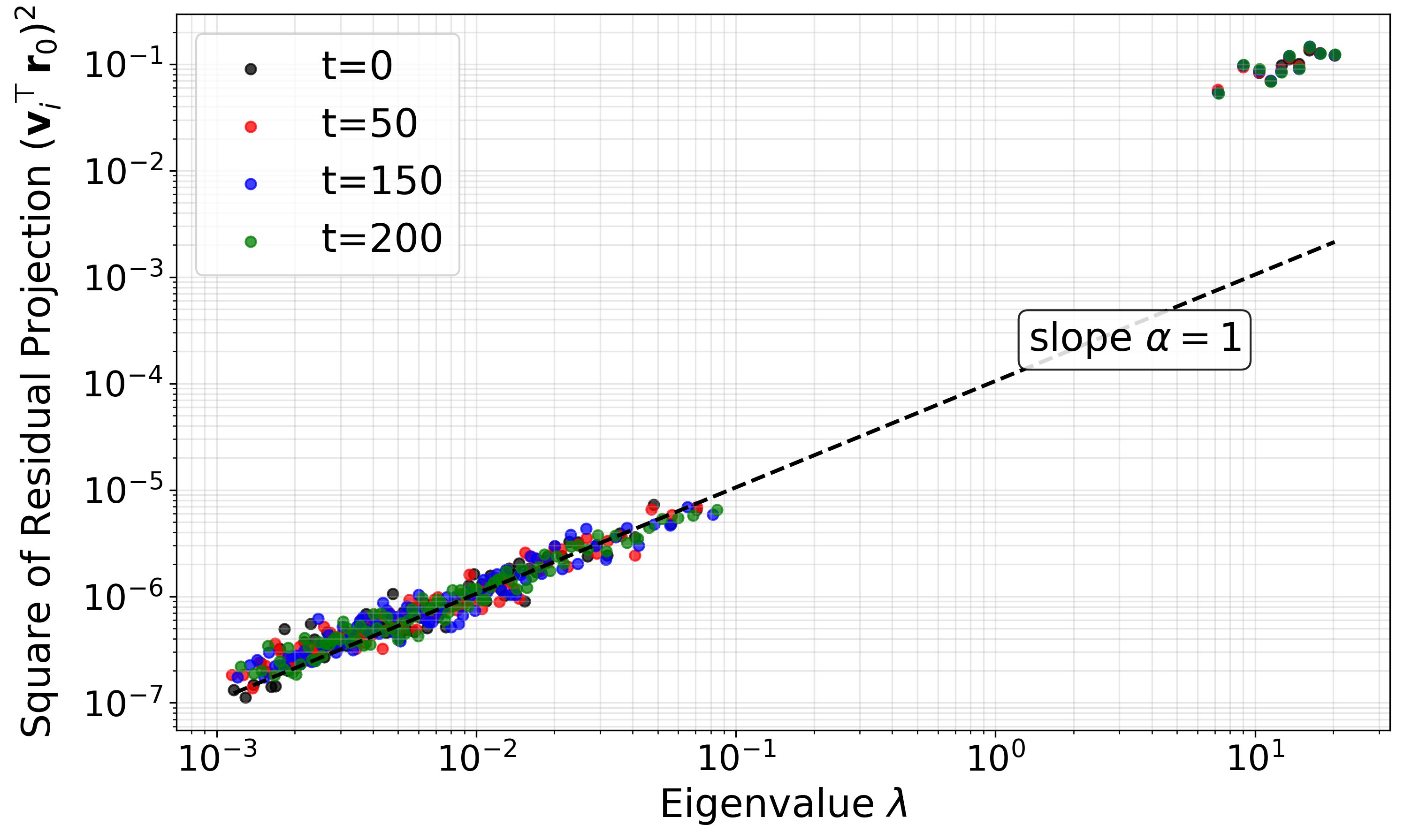}
        \vspace{-7pt}
        \label{fig:cnn_r0_projections_tinyimagenet10}
        \caption{Residual-NTK alignment(CNN,Tiny-ImageNet10)}
    \end{subfigure}
    \begin{subfigure}[t]{0.45\linewidth}
        \centering
        \includegraphics[width=\linewidth]{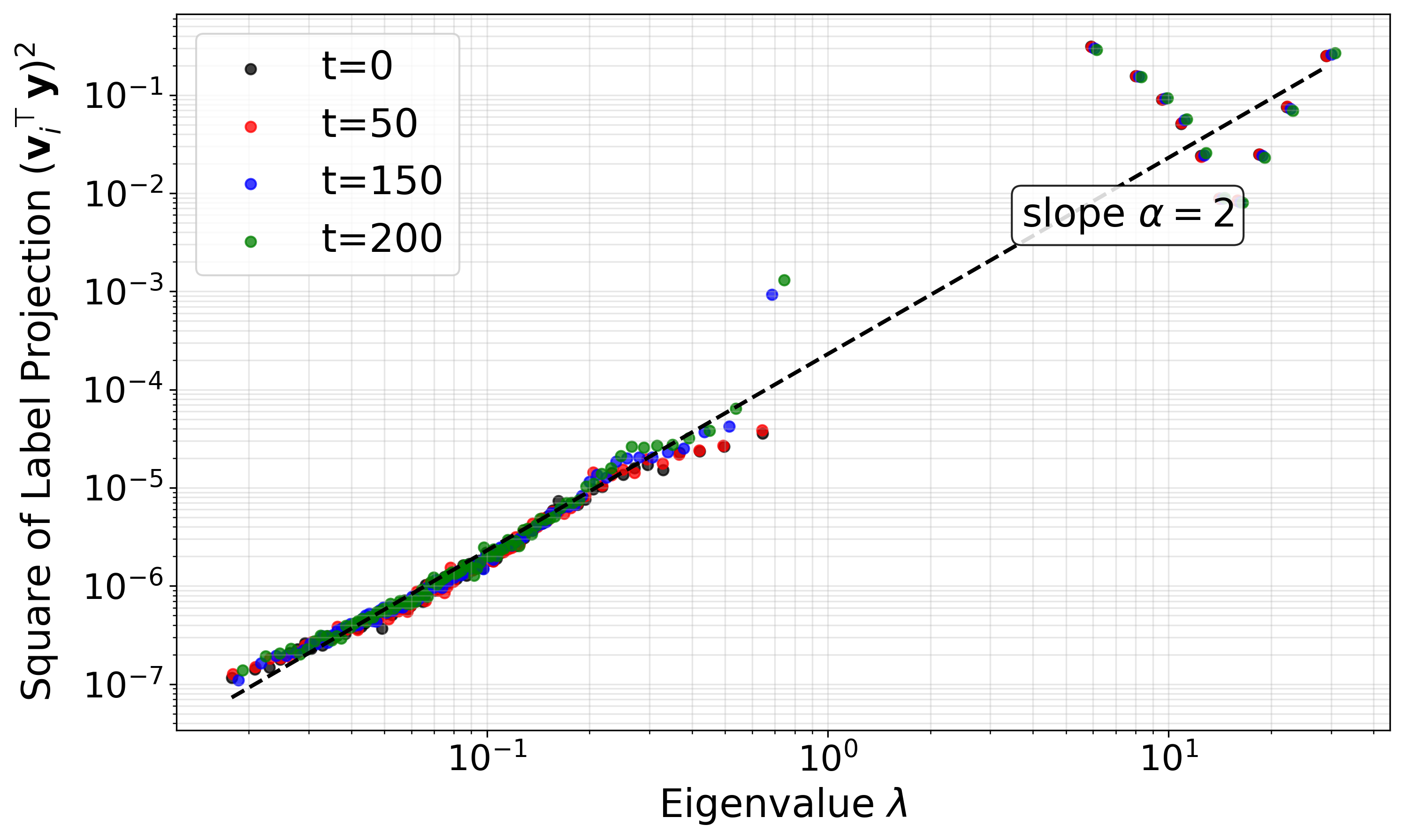}
        \vspace{-7pt}
        \label{fig:mlp_y_projections_svhn}
        \caption{Label-NTK alignment(MLP, SVHN)}
    \end{subfigure}
    \hfill
    \begin{subfigure}[t]{0.45\linewidth}
        \centering
        \includegraphics[width=\linewidth]{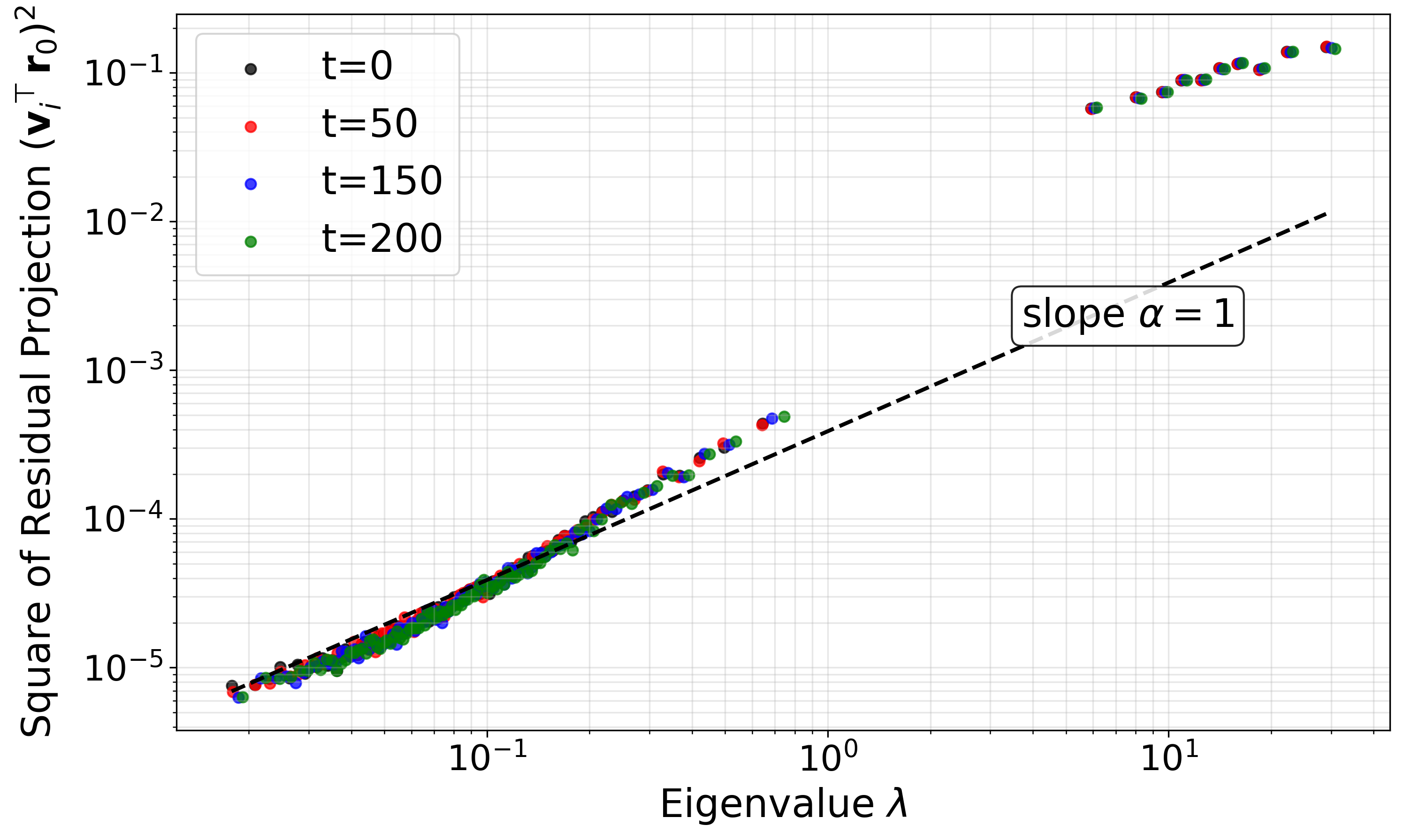}
        \vspace{-7pt}
        \label{fig:mlp_r0_projections_svhn}
        \caption{Residual-NTK alignment(MLP,SVHN)}
    \end{subfigure}
    \begin{subfigure}[t]{0.45\linewidth}
        \centering
        \includegraphics[width=\linewidth]{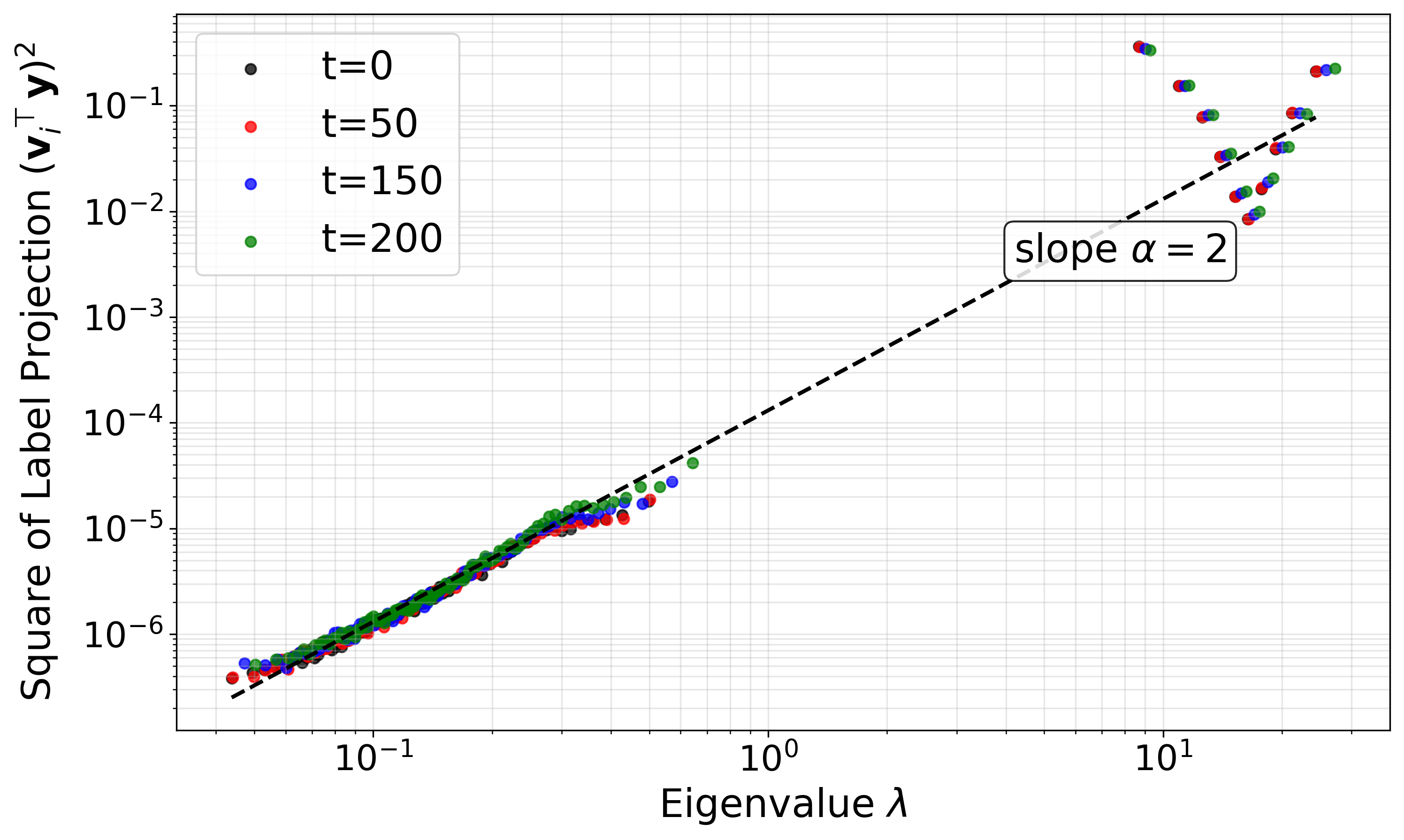}
        \vspace{-7pt}
        \label{fig:mlp_y_projections_fashionmnist}
        \caption{Label-NTK alignment(MLP, Fashion-MNIST)}
    \end{subfigure}
    \hfill
    \begin{subfigure}[t]{0.45\linewidth}
        \centering
        \includegraphics[width=\linewidth]{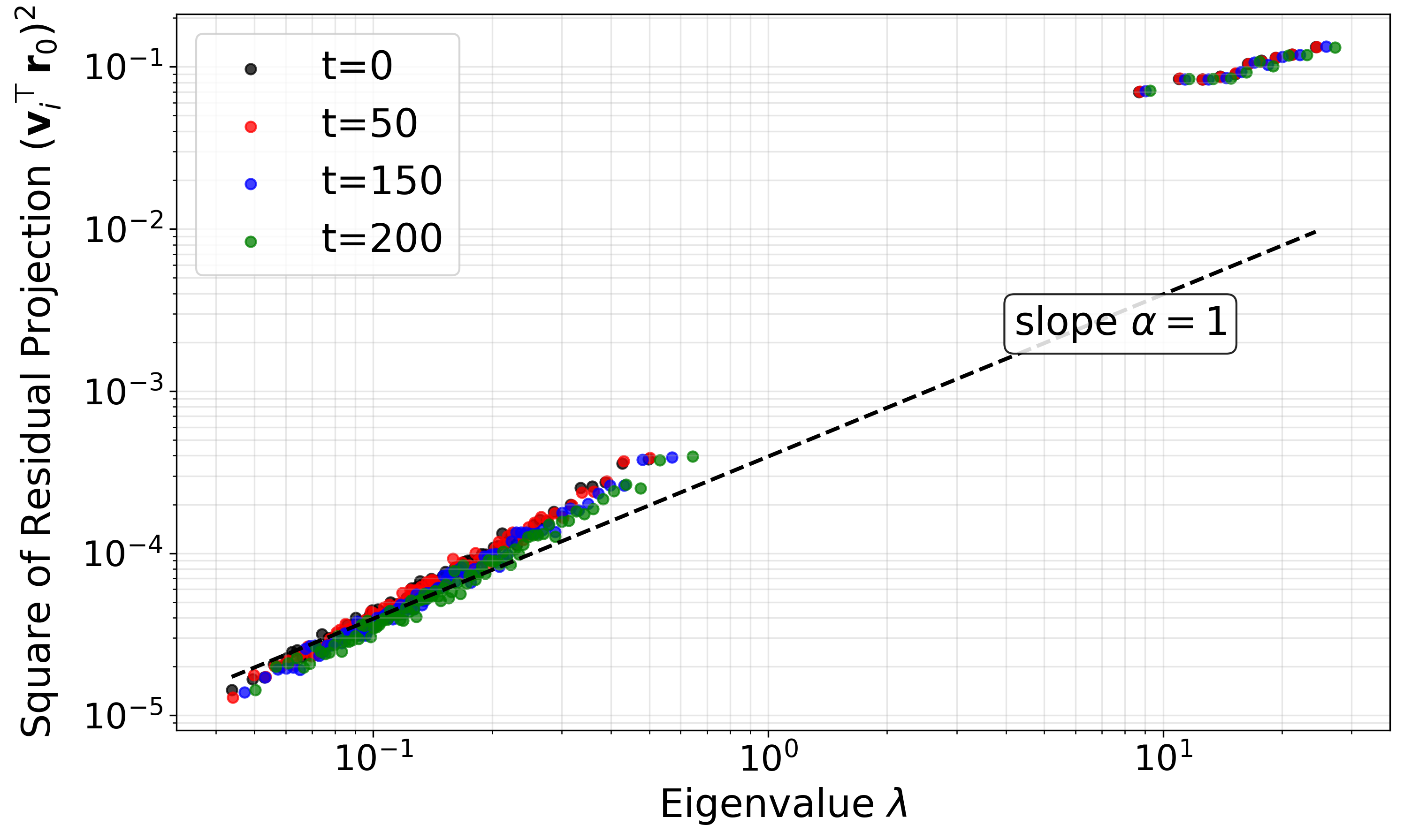}
        \vspace{-7pt}
        \label{fig:mlp_r0_projections_fashionmnist}
        \caption{Residual-NTK alignment(MLP,Fashion-MNIST)}
    \end{subfigure}
    \begin{subfigure}[t]{0.45\linewidth}
        \centering
        \includegraphics[width=\linewidth]{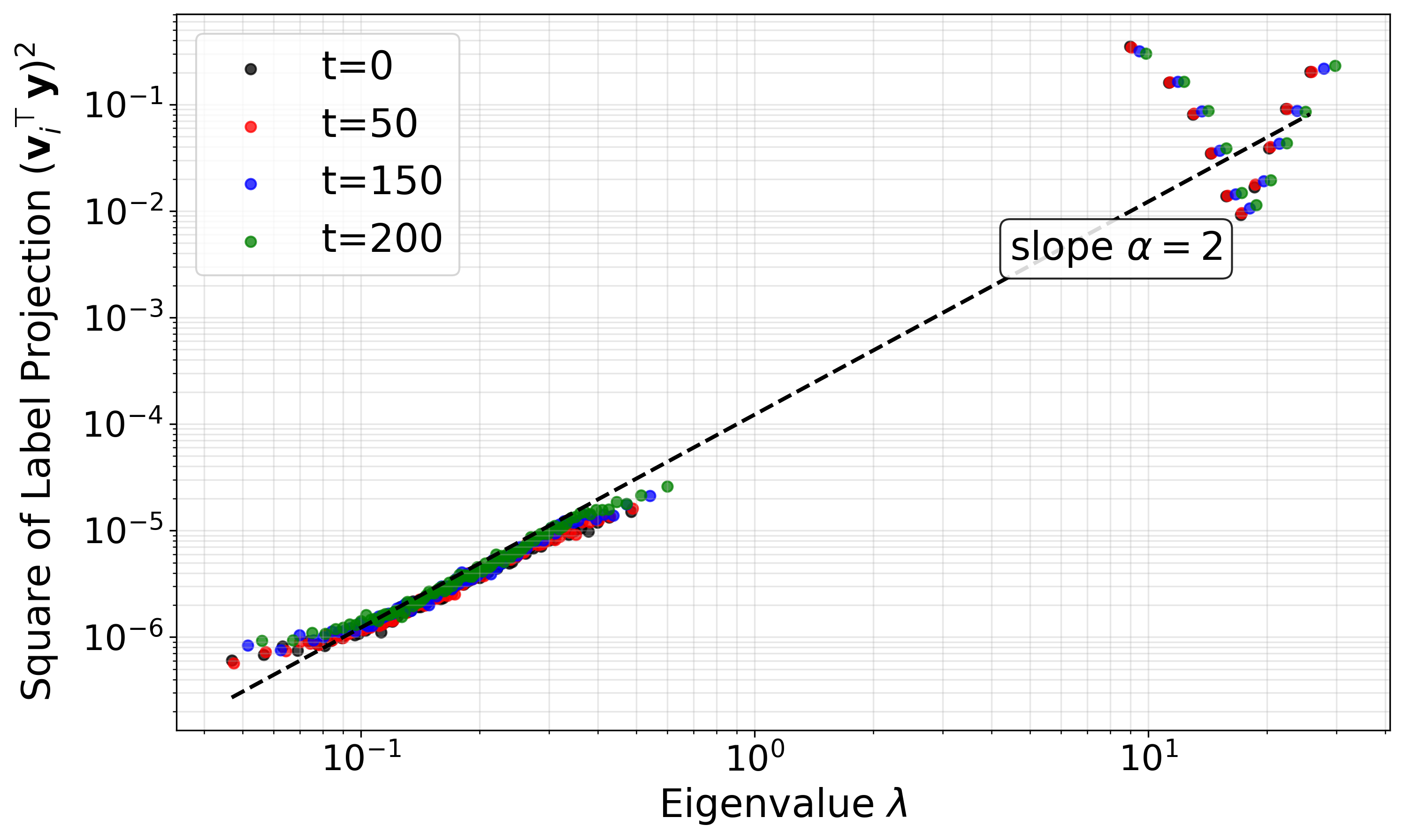}
        \vspace{-7pt}
        \label{fig:mlp_y_projections_mnist}
        \caption{Label-NTK alignment(MLP, MNIST)}
    \end{subfigure}
    \hfill
    \begin{subfigure}[t]{0.45\linewidth}
        \centering
        \includegraphics[width=\linewidth]{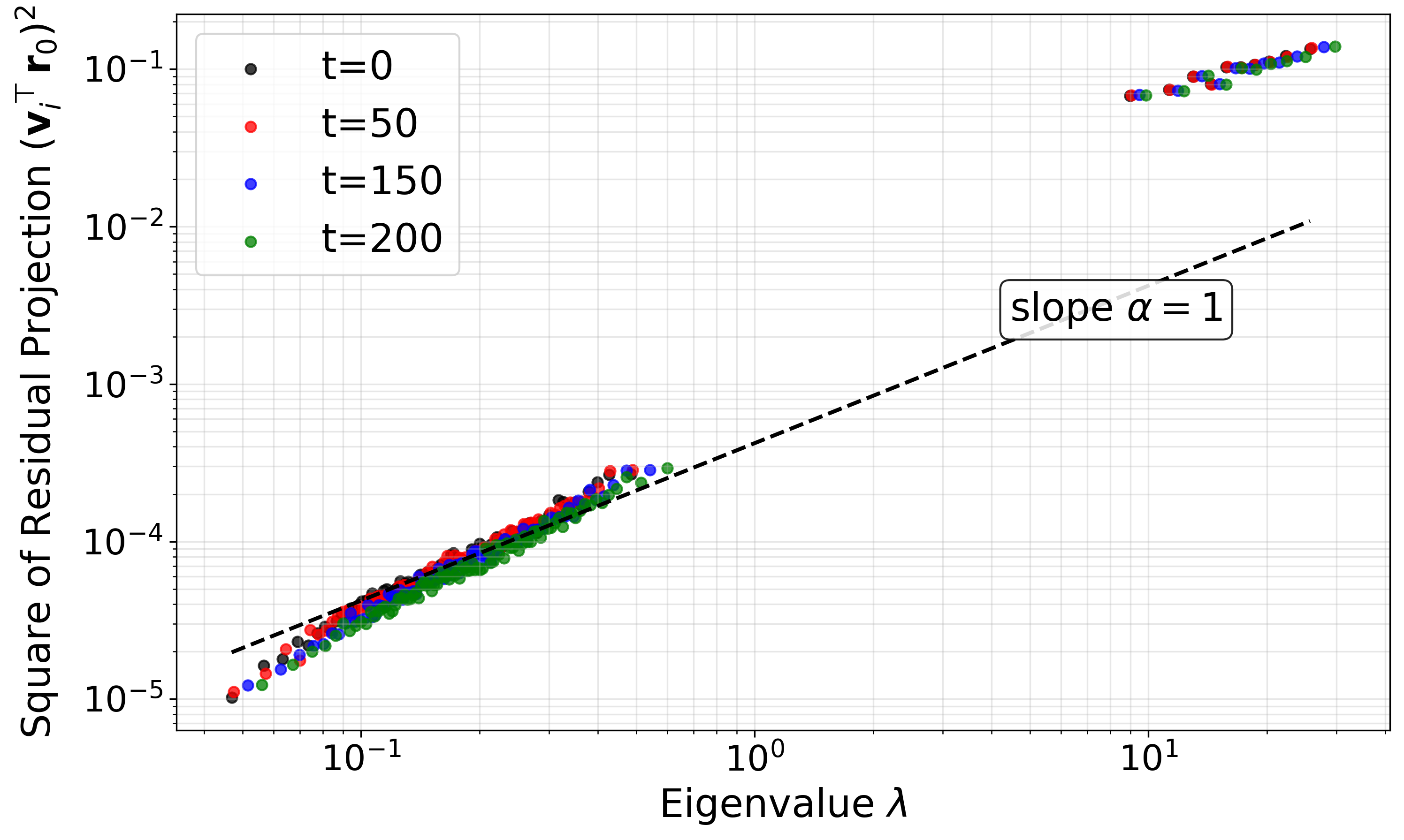}
        \vspace{-7pt}
        \label{fig:mlp_r0_projections_mnist}
        \caption{Residual-NTK alignment(MLP,MNIST)}
    \end{subfigure}
    \caption{{\it Label-NTK alignment} and {\it Residual-NTK alignment} for different architectures and datasets at multiple training time $t$. {\bf Left:} Label-NTK alignment. {\bf Right:} Residual-NTK alignment. {\bf First row:}  CNN  on the Tiny-ImageNet10;  {\bf Row 2-4:}  MLP  on SVHN, Fashion-MNIST, and MNIST respectively.  Each point is computed as an average over 500 batches, each containing 100 samples. Note: slope $\alpha=2$ (left) and $\alpha=1$ (right).}
    \label{fig:alignment_additional_plots}
\end{figure}


\clearpage

\end{document}